\documentclass[sigconf]{acmart}
\AtBeginDocument{%
  }
\usepackage{multirow} 
\usepackage{subcaption}
\usepackage{booktabs, tabularx, adjustbox, ltablex}
\usepackage{amsmath}

\newtheorem{theorem}{Theorem}[section] 

\usepackage{float}
\usepackage{algorithm}
\usepackage{algpseudocode}
\usepackage{multirow}
\usepackage{adjustbox}
\usepackage{microtype}
\usepackage{enumitem}
\usepackage{placeins}
\setcopyright{acmlicensed}
\copyrightyear{2026}
\acmYear{2026}
\acmDOI{XXXXXXX.XXXXXXX}
\acmConference[Conference acronym 'XX]{Make sure to enter the correct
  conference title from your rights confirmation email}{June 03--05,
  2026}{Woodstock, NY}
\acmISBN{978-1-4503-XXXX-X/2026/06}

\usepackage{hyperref}
\usepackage{url}
\usepackage{hyperref}
\usepackage{url}
\usepackage{pifont}
\usepackage{multirow} 
\usepackage{multirow} 
\usepackage{fancyhdr}
\usepackage{subcaption}

\begin{document}

\title[Diff-prior]{From Uniform to Learned Graph Priors: Diffusion for Structure Discovery
}

%


\author{Qi Shao}\authornotemark[1]
\email{shaoqi@seu.edu.cn}
\orcid{0000-0002-1975-7033}
\affiliation{%
  \institution{School of Mathematics, Southeast University}
  \city{Nanjing}
  \state{Jiangsu}
  \country{China}
}
\author{Hao Guo}
\email{213232865@seu.edu.cn}
\orcid{0009-0004-4526-5691}
\affiliation{%
  \institution{School of Mathematics, Southeast University}
  \city{Nanjing}
  \state{Jiangsu}
  \country{China}
}
\authornote{Qi Shao and Hao Guo contributed equally to this work.}

\author{Jiawen Chen}
\email{jiawenchen@seu.edu.cn}
\orcid{0009-0003-0530-0387}
\affiliation{%
  \institution{School of Mathematics, Southeast University}
  \city{Nanjing}
  \state{Jiangsu}
  \country{China}
}
\author{Duxin Chen}\authornotemark[2]
\email{chendx@seu.edu.cn}
\orcid{0000-0002-3194-2258}
\affiliation{%
  \institution{School of Mathematics, Southeast University}
  \city{Nanjing}
  \state{Jiangsu}
  \country{China}
}
\author{Wenwu Yu}
\email{wwyu@seu.edu.cn}
\orcid{0000-0003-3755-179X}
\affiliation{%
  \institution{School of Mathematics, Southeast University}
  \city{Nanjing}
  \state{Jiangsu}
  \country{China}
}
\authornote{Correspondence to Duxin Chen and Wenwu Yu.}

\renewcommand{\shortauthors}{Qi Shao, Hao Guo, Jiawen Chen, Duxin Chen, and Wenwu Yu}



 \copyrightyear{2026}
\acmYear{2026}
\setcopyright{cc}
\setcctype{by-nc-nd}
\acmConference[KDD '26]{Proceedings of the 32nd ACM SIGKDD Conference on Knowledge Discovery and Data Mining V.2}{August 09--13, 2026}{Jeju Island, Republic of Korea}
\acmBooktitle{Proceedings of the 32nd ACM SIGKDD Conference on Knowledge Discovery and Data Mining V.2 (KDD '26), August 09--13, 2026, Jeju Island, Republic of Korea}
\acmDOI{10.1145/3770855.3817940}
\acmISBN{979-8-4007-2259-2/2026/08}

\renewcommand{\shortauthors}{Trovato et al.}

\begin{abstract}
Neural relational inference (NRI) methods discover interaction graphs from trajectories through variational reasoning on discrete potential edges. However, these methods typically rely on oversimplified, factorized graph priors. Such priors, typically nearing uniform distributions, treat edges as independent entities. This systemic misalignment does not match the real-world systems and yields diffuse and indecisive edge posteriors limiting the reliability of structural discovery.  
To address this, we propose \textit{Diff-prior}, a diffusion-parameterized adaptive prior used to calibrate latent graph distribution rather than generate graphs. Our core insight is to reframe prior integration as a learnable denoising-style calibration that organizes scattered, uncertain edge posteriors into a more reliable overall structure which can be trained by the diffusion model. 
Diff-prior learns an adaptive structure prior that performs structured calibration on the edge posteriors during inference, guiding it towards a distribution closer to the underlying structure. 
The diff-prior operates before structural sampling and acts as a denoising calibrator directly on the encoder edge distribution, which provides a generic training paradigm over structured variables.   
Experiments on standard benchmarks validated our framework, and the results indicate that Diff-prior improves the performance of structure inference and generates more decisive edge posteriors across multiple NRI-family architectures.
The code is available on \url{https://github.com/Hardy158118/Diffprior}.
\end{abstract}





\maketitle

\section{Introduction}\label{}
In the observation of complex systems, dynamic information is usually obtained with the trajectories of system state, such as particle coordinates, multi-agent movements, and multivariate physiological signals. However, the underlying interaction graph structure of the system is often implicit \cite{Battaglia2016InteractionNetwork}. Neural Relational Inference aims to extract the "who influences whom" question from observational data, thereby constructing models that are not only predictive but also mechanistically interpretable \cite{kipf2018neural, webb2019factorised}.
Traditional NRI employs the classic path of "encoder inferring discrete edges, graph decoder reconstructing dynamics". However, this inference process is highly susceptible to the trap of uncertainty: Figure \ref{fig:figure1} shows the posterior edge probability distribution often exhibits excessive dispersion, leading to highly unstable reconstructed topology. This instability directly weakens the scientific persuasiveness of the model \cite{pan2024graph, manenti2024learning}.
\begin{figure}
  \centering
  \includegraphics[width=1\linewidth]{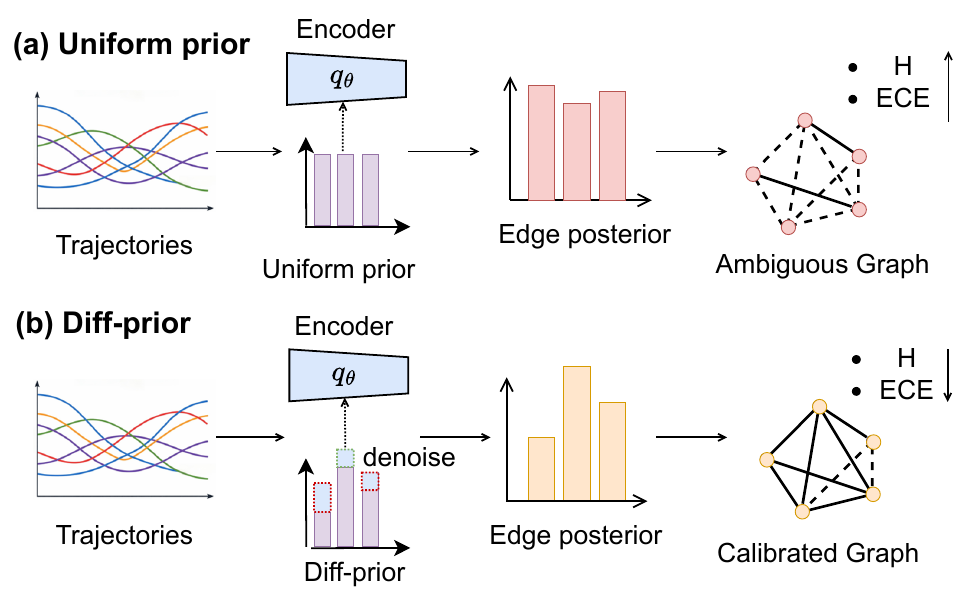}
\caption{\textbf{Diff-prior calibrates NRI graph inference.}
(a) A factorized prior can yield diffuse edge posteriors and an ambiguous graph.
(b) Diff-prior denoises and calibrates encoder logits before discretization, producing sharper edge posteriors and a more decisive graph.}
  \label{fig:figure1}
  \vspace{-20pt}
\end{figure}

The uncertainty of the prior distribution in graph structure inference cannot be entirely attributed to the optimization of discrete latent variables of the structure ~\cite{jang2017categorical,maddison2017concrete}, but rather points to deeper structural constraints in the regularization of graph uncertainty. To improve the scalability of structural distributions, NRI family methods typically employ factorized edge priors, which are usually instantiated as uniform categorical distributions. This assumption effectively deprives the model of its ability to understand graph-theoretic features such as sparsity, modularity and or motifs \cite{kipf2018neural,wang2024benchmarking}. When dynamic evidence cannot uniquely define the graph, this prior mismatch has low bias towards consistent configurations, leading to scattered and poorly calibrated posterior edge beliefs, ultimately reducing the robustness of the discovery, as quantified by expected calibration error (ECE) \cite{guo2017calibration} in Sec. ~\ref{subsec:diagnostics}.
This motivates us to learn a non-factorized structural prior that provides consistent structural bias while remaining end-to-end trainable. Unlike the independent priors commonly used in NRI-family models, our prior is defined over the full edge-logit configuration rather than as a set of independent edge-wise regularization terms.

Existing approaches to Neural Relational Inference primarily focus on dynamic modeling or integrating domain-specific constraints~\cite{wang2024benchmarking}. For instance, incorporating physical simulators or differentiable constraints can enhance predictive consistency~\cite{sanchez2018graph}. However, these methods predominantly optimize the decoder’s likelihood, prioritizing trajectory fitting over the acquisition of explicit, correlated priors at the graph structure level.
Another line of research explores the complex posterior inference mechanisms or designs dedicated time-series encoders (such as feature extractors based on Selective State-Space Models \cite{gu2024mamba}) to capture long-range dependencies or transient interaction patterns~\cite{wang2024structural, wang2024benchmarking}. Despite progress on specific tasks, these methods tend to be domain-dependent and lack a general, transferable mechanism to calibrate edge posterior beliefs and improve the stability of structure discovery.
These limitations raise a critical question: \textbf{can we learn a general structural prior that makes graph inference decisive?}

To address this gap, we propose Diff-prior, a differentiable framework that parameterizes a non-factorized structural prior via a generative diffusion process.
Departing from the conventional reliance on independent edge regularization, which inherently overlooks cross-edge dependencies, our approach pivots toward learning a prior capable of enforcing global graph consistency.
Instead of explicitly formulating an intractable distribution $p(G)$ over discrete graphs, we view the missing ingredient as a learnable calibration transformation defined on the full edge-logit tensor. Concretely, given the encoder's edge-type logits $\ell$, we seek a function $T_{\phi}$ that maps
an approximately independent logit configuration into a more coherent one, $\ell \xrightarrow{T_{\phi}} \tilde{\ell}$. The calibration target is the latent graph configuration as a whole, even though the denoising network can be implemented with a lightweight parameterization for scalability.
Hand-crafted regularizers (e.g., sparsity or degree penalties) that fail to encapsulate higher-order data-driven regularities, so $T_{\phi}$ leverages the generative power of diffusion models to autonomously recover the latent topology of the underlying system.

We reformulate the acquisition of $T_{\phi}$ as a denoising-style problem: the framework takes perturbed or structurally ambiguous representations as input and yields refined counterparts that align with learned topological motifs.
Leveraging the generative capacity of diffusion models,  
$T_{\phi}$ internalizes complex statistical dependencies among edges via iterative perturbation and recovery, it bypasses the need for manually engineered constraints.  
Importantly, our approach diverges from using diffusion as a standalone graph generator; instead, we deploy the learned denoising score as a prior-driven calibrator during inference. We then inject this prior {before} discretization.
Diff-prior operates directly on the encoder's edge-type logits as a stochastic calibrator, which nudges inference toward plausible configurations while preserving end-to-end trainability.
This logit-level interface avoids the information loss and brittleness of post-hoc corrections applied to sampled discrete graphs, and makes the module plug-and-play across NRI-style backbones. Our contributions can be summarized as follows:
\begin{itemize}
  \item \textbf{Method.} We introduce \textit{Diff-prior}, a diffusion-parameterized, {non-factorized} structural prior defined over the full edge-logit configuration. It is injected at the encoder {logit level} as a stochastic calibrator, enabling plug-and-play integration and joint end-to-end training with NRI-style dynamics models.
  \item \textbf{Mechanism.} We provide a diagnostic view of graph inference that links prior mismatch to {diffuse and miscalibrated} edge posteriors, and show that Diff-prior consistently produces {sharper and better-calibrated} edge beliefs, quantified by entropy and ECE.
  \item \textbf{Experiments.} Diff-prior demonstrates performance improvements across multiple benchmark tests using NRI-family methods, and ablation experiments confirm that these improvements stem from prior modeling rather than the number of parameters or training heuristics.
\end{itemize}

\section{Related Work}
\noindent
\textit{Latent graph inference from trajectories.}
Inferring latent interactions together with dynamics from observed trajectories is central to modeling complex systems.
Early works such as Interaction Networks model object--object interactions via message passing for dynamical prediction in physical domains \cite{Battaglia2016InteractionNetwork}.
Building on this paradigm, NRI formulates structure discovery as variational inference, where an encoder predicts discrete edge types from trajectories and a graph-conditioned decoder explains the dynamics under the latent graph \cite{kipf2018neural}.
This encoder--decoder template has become a standard backbone for relational reasoning in physical and multi-agent systems \cite{battaglia2018relational}.
Our work stays within this NRI-style inference setting, and focuses on improving how structural uncertainty is regularized through learned, correlated priors.

\textit{Improving dynamics and injecting physical knowledge.}
A large body of work enhances relational inference by strengthening the dynamics model or incorporating domain knowledge to improve prediction accuracy and physical consistency.
Graph Network-based Simulators improve generalization across system scales via particle-based message passing \cite{sanchez2020learning}, and physics-guided learning combines first-principles structure with data-driven components to improve interpretability and robustness \cite{yu2024learning}.
To better generalize across diverse physical substances, Graph-based Physics Engine  embeds material parameters as physical priors to model multiple substances across scenarios \cite{yang2021learning}.
Related efforts explore model construction for out-of-instance recovery \cite{li2022learning} and connections between learned simulators and particle-based simulation \cite{toshev2022relationships}.
While effective for dynamics prediction, these methods primarily act on the decoder or likelihood, and typically do not learn an explicit non-factorized prior over graph patterns, which is the motivation of our work.

\textit{Richer inference and temporal encoders for structure discovery.}
Beyond dynamics modeling, recent advances improve structure inference by developing richer posterior approximations or specialized temporal representations.
For example, selective state-space models with input-dependent transitions can handle irregular sampling, and the resulting temporal features can be combined with Graphical Flow Networks to approximate a posterior over structures \cite{wang2024structural}.
Dynamic NRI further extends the NRI framework to non-stationary settings by modeling time-varying relations and introducing a sequential relation prior over edges conditioned on history \cite{graber2020dynamic}.
Semi-implicit Neural ODE variants improve numerical stability and efficiency for partitioned dynamics \cite{zhang2025semi}, and multiscale modeling has been explored via diffusion autoencoders coupled with Graph Neural ODEs to capture cross-scale co-evolution \cite{Li3737087}.
Benchmarking studies further highlight that fair comparison is challenging due to domain-specific designs, and that robustness varies significantly across regimes and graph families \cite{wang2024benchmarking}.
Many of these methods still adopt factorized structural assumptions for scalability, while others focus on specialized temporal regimes; in contrast, we aim for a transferable, plug-and-play mechanism that learns a correlated structural bias to calibrate edge beliefs in NRI-style inference.

\textit{Diffusion models on discrete structures and graphs.}
Denoising diffusion models have become a powerful tool for learning complex distributions \cite{ho2020denoising,kong2023autoregressive,yang2024graphusion}.
For discrete variables, discrete diffusion formulations such as D3PM provide a principled framework for Markovian corruption and denoising beyond uniform transition kernels \cite{austin2021structured}.
In the graph domain, discrete graph diffusion models such as DiGress learn to denoise node and edge attributes under a discrete noising process \cite{vignacdigress}, and hybrid designs such as PARD combine autoregressive decomposition with diffusion-based denoising to improve efficiency while preserving permutation invariance \cite{zhao2024pard}.
Transformer-based denoisers further strengthen conditional control in graph diffusion \cite{liu2024graph}.
Most of these works study graph generation or conditional generation.
In contrast, our goal is structural inference in NRI-style models and we use diffusion to parameterize an adaptive structural prior and inject it at the encoder-logit interface to calibrate edge posteriors, enabling plug-and-play improvements across NRI-family backbones.

\section{Preliminaries}
\paragraph{Notations.}
We observe trajectories of $N$ entities over $T$ discrete time steps,
$\mathbf{X}=\{x_i(t)\in\mathbb{R}^d \mid i=1,\ldots,N;\, t=1,\ldots,T\}$, where $d$ is the state dimension.
Let $X_i=[x_i(1),\ldots,x_i(T)]$ and denote the collection by $\mathbf{X}$ (e.g., as a tensor in $\mathbb{R}^{N\times T\times d}$).

We assume a latent directed interaction graph with $K$ discrete edge types.
For each ordered pair $(i,j)\in\mathcal{E}$, we define a latent variable indicating the edge type from $i$ to $j$.
In our implementation, we use the fully connected directed set $\mathcal{E}=\{(i,j)\mid i,j\in[N]\}$, which includes self-loops; hence $|\mathcal{E}|=M=N^2$.

An encoder $f_{\mathrm{enc}}$ parameterized by $\theta_{\mathrm{enc}}$ maps $\mathbf{X}$ to edge-type logits
$\ell_{ij}=f_{\mathrm{enc}}(\mathbf{X})_{ij}\in\mathbb{R}^{K}$ for all $(i,j)\in\mathcal{E}$.
The decoder $f_{\mathrm{dec}}$ parameterized by $\theta_{\mathrm{dec}}$ models the dynamics of $\mathbf{X}$ conditioned on sampled edge $\mathcal{E}$.
\paragraph{Problem definition.}
Given trajectories $\mathbf{X}$, our goal is to infer the latent directed interaction graph by estimating the posterior of structure over edge types
\begin{equation}
q_{\theta}(Z\mid \mathbf{X}), \quad Z=\{\ell_{ij} \in \mathbb{R}^K\mid (i,j)\in\mathcal{E}\}.
\end{equation}
Here the corresponding categorical distribution is parameterized by \(\mathrm{softmax}(\ell_{ij})\), and NRI-family methods learn an encoder--decoder model by maximizing an ELBO with a graph prior $p(Z)$:
\begin{equation}
\log p_{\theta}(\mathbf{X}) \ge
\mathbb{E}_{q_{\theta}(Z\mid \mathbf{X})}\big[\log p_{\theta_{\mathrm{dec}}}(\mathbf{X}\mid Z)\big]
-\mathrm{KL}\!\left(q_{\theta}(Z\mid \mathbf{X})\,\|\,p(Z)\right),
\end{equation}
where $\log p_{\theta}(\mathbf{X})$ is the maximum likelihood of trajectory data $\mathbf{X}$, $\mathrm{KL}$ is the KL-divergence and $p(Z)$ is often used a uniform prior, which can leave many edges in a medium-confidence regime.
We therefore seek a plug-and-play adaptive prior that improves both structure inference and posterior reliability without changing the backbone encoder/decoder.

\begin{figure*}[h]
  \centering
  \includegraphics[width=1\linewidth]{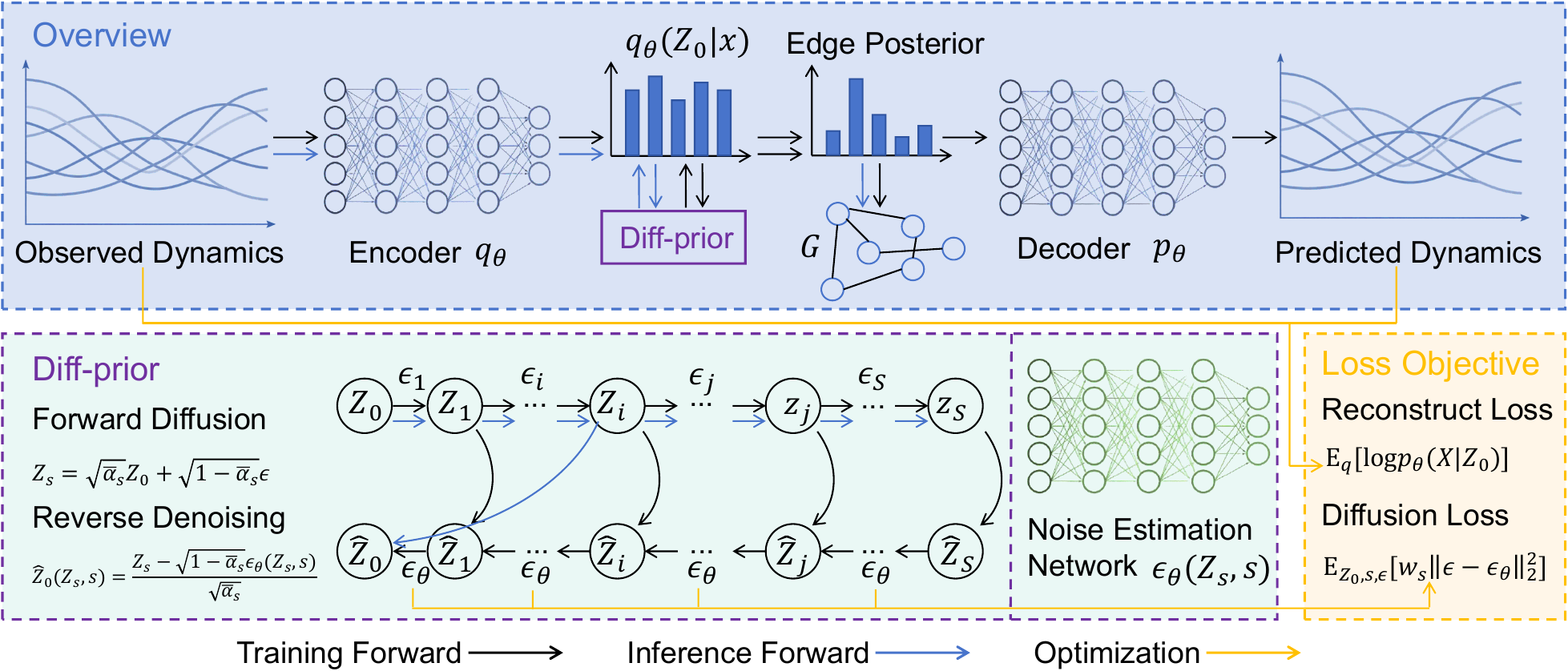}
  \vspace{-10pt}
  \caption{ Diff-prior overview and training objective. We learn a diffusion-parameterized, non-factorized prior over encoder logits $Z_0$ and inject it before discretization to calibrate NRI edge posteriors, then decode dynamics from the sampled graph. 
Training minimizes a reconstruction loss together with a diffusion denoising loss using a noise-estimation network.}
  \label{fig:2}
  \vspace{-10pt}
\end{figure*}

\section{Diffusion Prior for Neural Relational Inference}\label{sec4}
\subsection{Encoder: From data space to latent space}
At a high level, given the trajectories $\mathbf{x}=(\mathbf{x}^{1},\ldots,\mathbf{x}^{T})$, the encoder infers an approximate posterior distribution over latent interaction edge types logits $\ell_{ij}$. Since the ground-truth interaction graph is unknown, we operate on a fully connected graph and use a graph neural network (GNN) to predict the latent edge types \cite{kipf2018neural}:
\begin{equation}
\boldsymbol{\ell}=f_{\theta_{enc}}(\mathbf{x}),
\end{equation}
where $\boldsymbol{\ell}_{ij}\in\mathbb{R}^{K}$ are the logits over $K$ interaction types for edge $(i,j)$, obtained by an edge-to-node-to-edge message passing GNN encoder $f_{\theta_{enc}}$ followed by a final linear layer.

\subsection{Diffusion in latent space}

\subsubsection{Initial latent state and forward process for diffusion.}

Let $\{\ell_{ij}\in\mathbb{R}^{K}\}_{(i,j)}$ denote the edge logits produced by the encoder .
To obtain a continuous latent graph representation suitable for diffusion,
we define a Gaussianized variational posterior over the concatenated logit tensor
\begin{equation}
\label{eq:q_z0}
q_{\theta_{enc}}(Z_{0}\mid \mathbf{x})
= \mathcal{N}\!\bigl(Z_{0};\,\boldsymbol{\ell},\,\sigma_{\mathrm{enc}}^{2} I\bigr),
\qquad \boldsymbol{\ell}=f_{\theta_{enc}}(\mathbf{x}).
\end{equation}
where $\ell$ is the vector form of $\ell_{ij}$ and  $\sigma_{\mathrm{enc}}$ is a fixed, sufficiently small constant and $I$ is the identity matrix.
We employ the standard reparameterization trick,
\begin{equation}
\label{eq:reparam_z0}
Z_{0} = \boldsymbol{\ell} + \sigma_{\mathrm{enc}}\,\epsilon,\qquad \epsilon\sim\mathcal{N}(0,I).
\end{equation}
so that gradients can flow through sampling.

Based on continuous latent logits $Z_0$ , we define a discrete-time forward diffusion process over $S$ steps.
Let $\{\beta_s\}_{s=1}^{S}$ be a variance schedule, and denote $\alpha_s := 1-\beta_s$ and $\bar{\alpha}_s := \prod_{r=1}^{s}\alpha_r$.
The forward noising process is a Markov chain
{\small
\begin{equation}
\label{eq:forward_markov}
q(Z_{1:S}\mid Z_0) \;=\; \prod_{s=1}^{S} q(Z_s\mid Z_{s-1}),
q(Z_s\mid Z_{s-1}) \;=\; \mathcal{N}\!\bigl(Z_s;\sqrt{\alpha_s}\,Z_{s-1},\,\beta_s I\bigr).
\end{equation}
}
By construction, the corresponding marginal admits a closed form
\begin{equation}
\label{eq:forward_marginal}
q(Z_s\mid Z_0) \;=\; \mathcal{N}\!\bigl(Z_s;\sqrt{\bar{\alpha}_s}\,Z_0,\,(1-\bar{\alpha}_s)I\bigr),
\end{equation}
which yields the reparameterized sampling rule
\begin{equation}
\label{eq:forward_reparam}
Z_s \;=\; \sqrt{\bar{\alpha}_s}\,Z_0 + \sqrt{1-\bar{\alpha}_s}\,\epsilon,
\qquad
\epsilon \sim \mathcal{N}(0,I).
\end{equation}
Eq.~\eqref{eq:forward_reparam} enables efficient construction of noisy latents at an arbitrary diffusion step $s$ without unrolling the full chain, which we leverage throughout training and inference.\cite{ho2020denoising}

\subsubsection{Backward denoise process of the diffusion.}

We instantiate diffusion prior by a time-conditional denoiser $\epsilon_\theta(Z_s,s)$ that predicts Gaussian noise injected at the diffusion step $s$. In our implementation, we parameterize the denoiser with a lightweight
shared MLP for efficiency. The non-factorized effect comes from defining and optimizing the diffusion prior over the full edge-logit tensor before sampling, not
from explicitly modeling all pairwise edge interactions inside the denoiser.
This keeps the module scalable while moving the posterior beyond independent
edge-wise regularization toward a jointly calibrated structural state. We
also provide a Transformer-based denoiser with explicit non-factorized
modeling in Appendix~\ref{app4}.

Given $\epsilon_\theta$, one may equivalently form an estimate of the clean latent logits via
\begin{equation}
\label{eq:predict_z0}
\widehat{Z}_0(Z_s,s)
=
\frac{
Z_s - \sqrt{1-\bar{\alpha}_s}\,\epsilon_{\theta_{diff}}(Z_s,s)
}{
\sqrt{\bar{\alpha}_s}
}.
\end{equation}
Finally, we output refined logits through a residual update
\begin{equation}
\label{eq:one_step_refine}
Z_0^{\mathrm{ref}}
=
Z_0
+
\gamma\bigl(\widehat{Z}_0 - Z_0\bigr),
\qquad
\gamma\in(0,1],
\end{equation}
where $\gamma$ controls the refinement strength.
We feed $Z_0^{\mathrm{ref}}$ (rather than $Z_0$) into the decoder, which preserves an inexpensive and stable decoding interface.

\vspace{-5pt}
\subsection{Decoder: From latent space to data space}
After the encoder and the latent diffusion process, we obtain refined edge-type logits
$\tilde{\boldsymbol{\ell}}_{ij}$, stacked as $\tilde{\mathbf{Z}}_0$.
We convert them into edge-type probabilities via
$\boldsymbol{\pi}_{ij}=\mathrm{softmax}(\tilde{\boldsymbol{\ell}}_{ij})$,
and draw a discrete interaction graph $\mathbf{G}$ using the Gumbel--Softmax estimator
with temperature $\tau$, which enables end-to-end training through discrete edge samples.

Conditioned on $\mathbf{G}$, the decoder models one-step dynamics with a graph neural network and learns
\begin{equation}
p_{\theta_{\mathrm{dec}}}\!\left(\mathbf{x}^{t+1}\mid \mathbf{x}^{t},\mathbf{G}\right).
\end{equation}
We follow the NRI decoder design and employ an Interaction-Network-style message-passing architecture,
using edge-type-specific message functions to capture heterogeneous interaction mechanisms.
We use an isotropic Gaussian observation model with fixed variance, thereby turning likelihood maximization into a differentiable reconstruction term $\mathcal{L}_{\mathrm{rec}}$. The next subsection combines $\mathcal{L}_{\mathrm{rec}}$ with the regularization induced by the latent diffusion prior and derives the overall training objective.
\vspace{-8pt}
\subsection{Loss Function}
Our optimization couples the sequence model and the diffusion prior to the latent logits $Z_0$ through their shared dependency.
First, to learn the diffusion prior, we sample a timestep $s\sim\mathrm{Unif}\{1,2,\dots,S\}$ and noise $\epsilon\sim\mathcal{N}(0,I)$, construct $Z_s$ via Eq.~\eqref{eq:forward_reparam}.
Second, for decoder-facing reconstruction, we compute refined logits $Z_0^{\mathrm{ref}}$ using the single-step refinement in Eq.~\eqref{eq:one_step_refine} and feed $Z_0^{\mathrm{ref}}$ into the edge sampler and downstream decoder.
This design preserves a stable and inexpensive decoding interface by avoiding an explicit reverse diffusion chain during training and evaluation.

The overall objective is a weighted sum of the reconstruction term and the diffusion regularizer:
\begin{align}
\label{eq:total_objective}
\mathcal{L} &= \mathcal{L}_{\mathrm{rec}} \;+\; \mathcal{L}_{\mathrm{diff}}\\
&
=
\frac{1}{2\sigma^2}\,
\mathbb{E}_{q_{\theta_{\mathrm{enc}}}(Z_0\mid\mathbf{X})}
\Bigg[
\mathbb{E}_{q(\mathbf{z}\mid Z_0)}
\sum_{t=1}^{T-1}\sum_{j=1}^{N}
\bigl\|\mathbf{x}_{j}^{t+1}-\boldsymbol{\mu}_{j}^{t+1}\bigr\|_2^2
\Bigg]\\
&+
\mathbb{E}_{Z_0\sim q_{\theta_{\mathrm{enc}}}(Z_0\mid\mathbf{X})}
\;
\mathbb{E}_{s\sim \mathrm{Unif}(\{1,\ldots,S\}),\,\epsilon\sim\mathcal{N}(0,\mathbf{I})}
\Big[
w_s\,
\bigl\|\epsilon-\epsilon_{\theta_{\mathrm{diff}}}(Z_s,s)\bigr\|_2^2
\Big].
\end{align}
The concrete derivation and proof are provided in Sec.\ref{subsec:training_objective} and Appendix \ref{app:diffusion_kl_proofs}

\section{Optimization Objective}
\label{subsec:training_objective}
\subsection{Overview}
\label{sec:loss_overview}

Given the generative model specified in Sec.~\ref{sec4}, we introduce a continuous latent variable
$Z_0$ formed by concatenating the edge-type logits inferred from the observed trajectories $\mathbf{X}$.
Our model factorizes as
\begin{equation}
p(\mathbf{X},Z_0) \;=\; p_{\theta_{\mathrm{dec}}}(\mathbf{X}\mid Z_0)\, p_{\theta_{\mathrm{diff}}}(Z_0),
\end{equation}
where $p_{\theta_{\mathrm{diff}}}(Z_0)$ is an implicit prior induced by a latent-space diffusion process.
The marginal likelihood $\log p(\mathbf{X})=\log\int p_{\theta_{\mathrm{dec}}}(\mathbf{X}\mid Z_0)\,p_{\theta_{\mathrm{diff}}}(Z_0)\,dZ_0$
is intractable in general. We therefore resort to variational inference with an encoder
$q_{\theta_{\mathrm{enc}}}(Z_0\mid\mathbf{X})$ and maximize the evidence lower bound (ELBO):
\begin{equation}
\begin{aligned}
\log p(\mathbf{X}) &
\;\ge\;
\mathbb{E}_{q_{\theta_{\mathrm{enc}}}(Z_0\mid \mathbf{X})}
\big[\log p_{\theta_{\mathrm{dec}}}(\mathbf{X}\mid Z_0)\big]\\
&-
\mathrm{KL}\!\left(q_{\theta_{\mathrm{enc}}}(Z_0\mid \mathbf{X})\,\|\,p_{\theta_{\mathrm{diff}}}(Z_0)\right).
\end{aligned}
\end{equation}
Unlike the conventional NRI prior that decomposes over edges,
$p_{\theta_{\rm diff}}(Z_0)$ is defined on the concatenated latent graph
representation $Z_0$. The resulting KL term regularizes the inferred
structure at the level of the latent graph configuration, rather than
penalizing each edge posterior independently.
Equivalently, we minimize the negative ELBO:
\begin{equation}
\mathcal{L}_{\mathrm{ELBO}}(\mathbf{X})
=
\mathcal{L}_{\mathrm{rec}}(\mathbf{X})
+
\mathcal{L}_{\mathrm{diff}}^{\mathrm{marg}}(\mathbf{X}),
\end{equation}
where $\mathcal{L}_{\mathrm{rec}}(\mathbf{X})
:=
-\,\mathbb{E}_{q_{\theta_{\mathrm{enc}}}(Z_0\mid \mathbf{X})}
[\log p_{\theta_{\mathrm{dec}}}(\mathbf{X}\mid Z_0)]$
is the reconstruction term and
$\mathcal{L}_{\mathrm{diff}}^{\mathrm{marg}}(\mathbf{X})
:=
\mathrm{KL}\!\left(q_{\theta_{\mathrm{enc}}}(Z_0\mid \mathbf{X})\,\|\,p_{\theta_{\mathrm{diff}}}(Z_0)\right)$
regularizes the inferred interaction structure.

A common approach motivates diffusion priors via ESM--DSM equivalences and drops terms treated as constants \cite{yang2024graphusion}.
In a learned latent space, however, the marginal
\begin{equation}
q(Z_s)=\int q(Z_s\mid Z_0)\,q_{\theta_{diff}}(Z_0)\,dZ_0
\end{equation}
depends on the encoder and a learnable prior, so these terms are not constant under joint optimization.
The resulting intractability of the required marginal-score terms is discussed in Appendix~\ref{app:latent_score_matching}.

We therefore avoid this simplification and instead optimize a joint-KL upper bound (Thm.~\ref{thm:marginal_joint_kl}), which leads to a tractable diffusion decomposition.

\subsection{Reconstruction Loss}\label{rec_loss}
We employ a Gaussian observation model for each node and time step $
p(\mathbf{x}^{t+1}_j \mid \mathbf{x}^{t}, \mathbf{G})
=
\mathcal{N}\!\left(\boldsymbol{\mu}^{t+1}_j,\; \sigma^2 \mathbf{I}\right)$,
where $\boldsymbol{\mu}^{t+1}_j$ is produced by the GNN decoder conditioned on the graph $\mathbf{G}$,
and $\sigma^2$ is a fixed variance. Up to an additive constant, the negative log-likelihood reduces to a scaled MSE objective:
\begin{equation}
\mathcal{L}_{\mathrm{rec}}(\mathbf{X})
=
\frac{1}{2\sigma^2}\,
\mathbb{E}_{q_{\theta_{\mathrm{enc}}}(Z_0\mid\mathbf{X})}
\Bigg[
\mathbb{E}_{q(\mathbf{z}\mid Z_0)}
\sum_{t=1}^{T-1}\sum_{j=1}^{N}
\bigl\|\mathbf{x}_{j}^{t+1}-\boldsymbol{\mu}_{j}^{t+1}\bigr\|_2^2
\Bigg].
\end{equation}
In practice, we approximate the expectations using one sample of $\mathbf{z}$
(e.g., via Gumbel-Softmax) per mini-batch and report the mean squared error averaged over nodes and time steps.

\subsection{Diffusion Loss}
\label{sec:diff_reg}

Recall that our diffusion module induces an implicit prior over the continuous logit latent
$Z_0\in\mathbb{R}^{M\times K}$, denoted by $p_{\theta_{\mathrm{diff}}}(Z_0)$.
Let $S$ be the number of diffusion steps.
We consider a fixed forward noising process
$q(Z_{1:S}\mid Z_0)=\prod_{s=1}^{S} q(Z_s\mid Z_{s-1})$,
where each transition is Gaussian with a pre-defined variance schedule,
and a learned reverse process
$p_{\theta_{\mathrm{diff}}}(Z_{0:S})=p(Z_S)\prod_{s=1}^{S} p_{\theta_{\mathrm{diff}}}(Z_{s-1}\mid Z_s)$,
with $p(Z_S)=\mathcal{N}(0,\mathbf{I})$.

\paragraph{Intractable marginal KL}
The diffusion regularizer is
\begin{equation}
\mathcal{L}_{\mathrm{diff}}^{\mathrm{marg}}(\mathbf{X})
=
\mathrm{KL}\!\left(q_{\theta_{\mathrm{enc}}}(Z_0\mid\mathbf{X}) \,\|\, p_{\theta_{\mathrm{diff}}}(Z_0)\right),
\label{eq:ldiff_marginal}
\end{equation}
which is generally intractable because $p_{\theta_{\mathrm{diff}}}(Z_0)$ is defined only implicitly as the
marginal of the diffusion chain.

\paragraph{A tractable upper bound via joint augmentation.}
We introduce the augmented variational distribution
\begin{equation}
Q_{\theta_{\mathrm{enc}}}(Z_{0:S}\mid \mathbf{X})
\;:=\;
q_{\theta_{\mathrm{enc}}}(Z_0\mid\mathbf{X})\, q(Z_{1:S}\mid Z_0),
\label{eq:aug_posterior}
\end{equation}
and the diffusion joint 
\begin{equation}
P_{\theta_{\mathrm{diff}}}(Z_{0:S})
\;:=\;
p_{\theta_{\mathrm{diff}}}(Z_{0:S}).
\label{eq:diff_joint}
\end{equation}

\begin{theorem}[Marginal-to-joint KL bound]
\label{thm:marginal_joint_kl}
For any $\mathbf{X}$, the marginal KL in Eq.~\eqref{eq:ldiff_marginal} is upper bounded by
\begin{equation}
\mathrm{KL}\!\left(q_{\theta_{\mathrm{enc}}}(Z_0\mid\mathbf{X}) \,\|\, p_{\theta_{\mathrm{diff}}}(Z_0)\right)
\;\le\;
\mathrm{KL}\!\left(Q_{\theta_{\mathrm{enc}}}(Z_{0:S}\mid\mathbf{X}) \,\|\, P_{\theta_{\mathrm{diff}}}(Z_{0:S})\right).
\label{eq:ldiff_upper}
\end{equation}
\end{theorem}
\noindent
We denote the right-hand side of Eq.~\eqref{eq:ldiff_upper} by $\widetilde{\mathcal{L}}_{\mathrm{diff}}(\mathbf{X})$.
The inequality follows from the data-processing inequality for KL under marginalization.
The complete proof is provided in Appendix~\ref{app:proof_marginal_joint_kl}.

\paragraph{Connection to the diffusion variational objective.}
The surrogate $\widetilde{\mathcal{L}}_{\mathrm{diff}}(\mathbf{X})$ admits a standard diffusion-style decomposition
into a sum of per-step terms, each of which is tractable for Gaussian transitions:
\begin{equation}
\begin{aligned}
\widetilde{\mathcal{L}}_{\mathrm{diff}}(\mathbf{X})
&=
\mathbb{E}_{q_{\theta_{\mathrm{enc}}}(Z_0\mid\mathbf{X})}
\Big[
-\mathrm{ELBO}_{\theta_{\mathrm{diff}}}(Z_0)
\Big], \\
-\mathrm{ELBO}_{\theta_{\mathrm{diff}}}(Z_0)
&=
L_S + \sum_{s=2}^{S} L_{s-1} + L_0 .
\end{aligned}
\label{eq:ddpm_elbo_decomp}
\end{equation}
where each $L_{s}$ is a KL between Gaussians comparing the forward posterior and the learned reverse transition,and the specific form of $L_s,s \in \{ 0,1,...S\}$ is provided in Appendix \ref{app:proof_C3}.
Derivations and explicit forms are given in Appendix~\ref{app:proof_C2}--\ref{app:proof_C3}.

In the next section, we further present an efficient single-timestep Monte Carlo estimator for optimizing
$\widetilde{\mathcal{L}}_{\mathrm{diff}}(\mathbf{X})$ in practice.

\subsection{Efficient Optimization via Single-Timestep Sampling}
\label{sec:efficient_diff}

Eq.~\eqref{eq:ddpm_elbo_decomp} expresses the diffusion surrogate
as a sum of $O(S)$ per-step terms. Computing all steps for every minibatch is unnecessary.
Instead, we optimize the diffusion regularizer using a single-timestep Monte Carlo estimator,
following the standard practice in diffusion model training.

\paragraph{Single-timestep estimator.}
From Eq.~\eqref{eq:ddpm_elbo_decomp}, the intermediate sum can be rewritten as an expectation:
\begin{equation}
\sum_{s=2}^{S} L_{s-1}(Z_0)
=
(S-1)\;\mathbb{E}_{s\sim \mathrm{Unif}(\{1,\ldots,S\})}\!\left[L_{s-1}(Z_0)\right].
\end{equation}
Therefore, we estimate $-\mathrm{ELBO}_{\theta_{\mathrm{diff}}}(Z_0)$ by sampling a single timestep
$s\sim\mathrm{Unif}(\{1,\ldots,S\})$ per iteration, which reduces the diffusion-loss cost from $O(S)$ to $O(1)$.

\paragraph{Noise-prediction form.}
We adopt the common $\epsilon$-prediction parameterization for the reverse process.
The forward marginal admits the closed-form reparameterization
\begin{equation}
Z_s
=
\sqrt{\bar\alpha_s}\,Z_0
+
\sqrt{1-\bar\alpha_s}\,\epsilon,
\qquad
\epsilon\sim\mathcal{N}(0,\mathbf{I}).
\label{eq:zt_reparam}
\end{equation}
The denoiser $\epsilon_{\theta_{\mathrm{diff}}}(Z_s,s)$ predicts $\epsilon$ from $(Z_s,s)$.
As shown in Appendix~\ref{app:proof_C3}, the per-step KL term is equivalent (up to a known weighting and additive constants)
to a weighted noise regression loss. Concretely, we optimize
{\small
\begin{equation}
\mathcal{L}_{\mathrm{diff}}(\mathbf{X})
=
\mathbb{E}_{Z_0\sim q_{\theta_{\mathrm{enc}}}(Z_0\mid\mathbf{X})}
\;
\mathbb{E}_{s\sim \mathrm{Unif}(\{1,\ldots,S\}),\,\epsilon\sim\mathcal{N}(0,\mathbf{I})}
\Big[
w_s\,
\bigl\|\epsilon-\epsilon_{\theta_{\mathrm{diff}}}(Z_s,s)\bigr\|_2^2
\Big],
\label{eq:diff_train_loss}
\end{equation}
}
where
\begin{equation}
w_s \;:=\; \frac{\beta_s^2}{2\,\sigma_s^2\,\alpha_s\,(1-\bar{\alpha}_s)},
\quad
\alpha_s := 1-\beta_s,\quad \bar{\alpha}_s := \prod_{r=1}^s \alpha_r,
\end{equation}
and $\sigma_s^2$ is the reverse-process variance at step $s$.
The derivation of this weighting from the per-step Gaussian KL terms is provided in Appendix~\ref{app:proof_C4}.
This single-timestep estimator is unbiased for the corresponding expectation objective, which is proved in Appendix~\ref{app:proof_C5}.

\subsection{Overall Objective}
\label{sec:overall_objective}

Plugging the estimator in Sec.\ref{sec:efficient_diff} into the diffusion regularizer, we minimize the following
empirical objective over the training set $\mathcal{D}$:
\begin{equation}
\min_{\theta_{\mathrm{enc}},\,\theta_{\mathrm{diff}},\,\theta_{\mathrm{dec}}}
\;
\mathbb{E}_{\mathbf{X}\sim\mathcal{D}}
\Big[
\mathcal{L}_{\mathrm{rec}}(\mathbf{X})
+
\mathcal{L}_{\mathrm{diff}}(\mathbf{X})
\Big],
\label{eq:final_objective}
\end{equation}
where $\mathcal{L}_{\mathrm{rec}}$ is the Gaussian reconstruction loss in Sec.\ref{rec_loss} and
$\mathcal{L}_{\mathrm{diff}}$ is the single-timestep diffusion estimator in Eq.\ref{eq:diff_train_loss}.
Algorithm~\ref{alg:joint_train_diffvae} summarizes the optimization procedure.

\begin{algorithm}[ht]
\caption{Joint training with a latent diffusion prior.}
\label{alg:joint_train_diffvae}
\begin{algorithmic}[1]
\Require minibatch trajectories $\mathbf{x}$;
encoder $q_{\theta_{\mathrm{enc}}}$;
decoder $p_{\theta_{\mathrm{dec}}}$;
denoiser $\epsilon_{\theta_{\mathrm{diff}}}(\cdot,\cdot)$;
diffusion steps $S$ and schedule $\{\beta_s\}_{s=1}^{S}$ with $\alpha_s=1-\beta_s$, $\bar\alpha_s=\prod_{r=1}^{s}\alpha_r$;
$\sigma_{\mathrm{enc}}$ ;
refinement strength $\gamma$;

\Ensure updated parameters $(\theta_{\mathrm{enc}},\theta_{\mathrm{dec}},\theta_{\mathrm{diff}})$.

\Function{RefineOneStep}{$Z_0,\,\,\gamma,\,\epsilon$}
    \State $s_{ref} \sim Unif(1,2,...S)$
  \State $Z_{s_{\mathrm{ref}}} \gets \sqrt{\bar\alpha_{s_{\mathrm{ref}}}}\,Z_0 + \sqrt{1-\bar\alpha_{s_{\mathrm{ref}}}}\,\epsilon$
  \State $\widehat{\epsilon} \gets \epsilon_{\theta_{\mathrm{diff}}}(Z_{s_{\mathrm{ref}}},\,s_{\mathrm{ref}})$
  \State $\widehat{Z}_0 \gets \bigl(Z_{s_{\mathrm{ref}}} - \sqrt{1-\bar\alpha_{s_{\mathrm{ref}}}}\,\widehat{\epsilon}\bigr)\big/ \sqrt{\bar\alpha_{s_{\mathrm{ref}}}}$
  \State \Return $Z_0 + \gamma\bigl(\widehat{Z}_0 - Z_0\bigr)$
\EndFunction

\State $\boldsymbol{\ell} \gets f_{\theta_{\mathrm{enc}}}(\mathbf{x})$ 
\State sample $\epsilon^{ref},\epsilon_0\sim\mathcal{N}(0,I)$
\State $Z_0 \gets \boldsymbol{\ell} + \sigma_{\mathrm{enc}}\,\epsilon_0$ 
\State $Z_0^{\mathrm{ref}} \gets$ \Call{RefineOneStep}{$Z_0,\,\gamma,\,\epsilon^{ref}$}

\State sample edges $G \sim \mathrm{GumbelSoftmax}(Z_0^{\mathrm{ref}})$
\State $\mathcal{L}_{\mathrm{rec}} \gets -\log p_{\theta_{\mathrm{dec}}}(\mathbf{x}\mid G)$

\State sample $s \sim \mathrm{Unif}\{1,2,\dots,S\}$ and $\epsilon\sim\mathcal{N}(0,I)$
\State $Z_s \gets \sqrt{\bar\alpha_s}\,Z_0 + \sqrt{1-\bar\alpha_s}\,\epsilon$
\State $\widehat{\epsilon} \gets \epsilon_{\theta_{\mathrm{diff}}}(Z_s,s)$
\State ${\mathcal{L}}_\mathrm{diff}(\mathbf{X}) \gets  w_s \lVert \epsilon - \widehat{\epsilon}\rVert_2^2$

\State $\mathcal{L} \gets \mathcal{L}_{\mathrm{rec}} + {\mathcal{L}}_{diff}(\mathbf{X})$
\State update $(\theta_{\mathrm{enc}},\theta_{\mathrm{dec}},\theta_{\mathrm{diff}})$ by one optimizer step on $\nabla \mathcal{L}$
\end{algorithmic}
\end{algorithm}



\begin{table*}[htbp]
\caption{Main results (AUROC, \%). Diff-prior is plugged into different backbones across network families and dynamics. \textbf{Bold} indicates the best performance.}
\label{tab:table1}
\begin{tabular}{c|c|c|ccccc|cc}
\hline\hline
Dynamics  & Model  & Prior   & {VN\_15} & {FW\_15} & {BN\_15} & {GRN\_15} & {CRNA\_15} & {AVG.}  & {Imp.}  \\ \hline
\multirow{8}{*}{Spring} & \multirow{3}{*}{NRI} & Uniform & 95.51±0.60   & 81.09±0.02   & 99.65±0.00   & 89.94±2.41  & 83.91±0.53  & 90.02  & 0.00 \\
  &  & Fixed   & \textbf{95.71±0.47}  & 81.01±0.03   & \textbf{99.67±0.00}  & 90.22±1.99  & 84.36±0.32   &   90.19  &   +0.17  \\
  &  & Diff-prior  & 95.17±0.13   & \textbf{81.51±0.14}  & 99.52±0.02   & \textbf{90.58±1.67}   & \textbf{85.65±1.28}  & \textbf{90.49} & \textbf{+0.47}  \\ \cline{2-10} 
  & \multirow{3}{*}{ACD} & Uniform & 94.31±0.14  & 81.58±0.02  & 99.69±0.02  & 89.52±0.23   & 85.9 ±0.43  &  90.20   &  0.00   \\
  &  & Fixed   & 94.38±0.18   & 81.49±0.09  & 99.71±0.01  & 87.53±0.06  & 83.47±0.59   &  89.31   &  -0.89   \\
  &  & Diff-prior  & \textbf{95.96±0.24}   & \textbf{81.73±0.00}   & \textbf{99.76±0.00}   & \textbf{91.19±0.05}   & \textbf{90.15±0.37}  & \textbf{91.17} & \textbf{+0.97} \\ \cline{2-10} 
  & \multirow{2}{*}{MPM}  & Uniform   & 98.25±0.09  & 79.67±0.26   & \textbf{100±0.00}   & 92.31±0.28   & 85.75±0.91  &   91.20  &  0.00  \\
  &  & Diff-prior  & \textbf{98.44±0.03}  & \textbf{80.5±0.03}  & 99.99±0.01   & \textbf{93.49±0.19}  & \textbf{89.88±0.66}   & \textbf{92.46}  &  \textbf{+1.26}   \\ \hline\hline
\multirow{8}{*}{Netsim} & \multirow{3}{*}{NRI} & Uniform & 86.17±0.68   & 51.78±0.02   & \textbf{99.88±0.00}  & 77.87±0.31  & 50.53±0.11  & 73.25 & 0.00  \\
  &  & Fixed   & 88.75±7.57   & 50.60±0.00   & 99.68±0.00   & \textbf{83.7±0.28}  & 51.61±0.13   &  74.87   &   +1.62  \\
  &  & Diff-prior  & \textbf{91.5±0.33}  & \textbf{55.73±0.02}  & 99.75±0.08   & 77.38±0.60  & \textbf{51.80±0.65}  &  \textbf{75.23}   &  \textbf{+1.98}   \\ \cline{2-10} 
  & \multirow{3}{*}{ACD} & Uniform & \textbf{93.91±0.75}   & 51.79±0.03  & 99.85±0.00  & 69.85±0.21   & 51.74±0.66  &  73.43   &  0.00 \\
  &  & Fixed   & 93.62±0.48  & 50.03±0.06  & 99.83±0.01  & 69.99±0.17   & 60.70±0.13   &  74.83   &+1.40   \\
  &  & Diff-prior  & 92.15±0.52  & \textbf{52.31±0.08}   & \textbf{99.88±0.01}   & \textbf{80.98±0.63}  & \textbf{62.28±0.11}   &   \textbf{77.52}  &   \textbf{+4.09}  \\ \cline{2-10} 
  & \multirow{2}{*}{MPM} & Uniform   & 81.45±0.28  & 52.17±0.19  & \textbf{99.95±0.03}  & 70.46±0.77   & 50.28±0.06  &  70.86   &   0.00  \\
  &  & Diff-prior  & \textbf{96.65±0.32}  & \textbf{52.98±0.66}  & 99.75±0.08   & \textbf{84.51±3.66}   & \textbf{52.56±0.04}   &   \textbf{77.29}  &   \textbf{+6.43}  \\ \hline\hline
\end{tabular}
\end{table*}

\section{Experiments}

\subsection{Experiments Setup}
We evaluate our method on the public benchmarks from the StructInfer suite\footnote{\url{https://structinfer.github.io/methods/}}, and re-run all experiments to obtain consistent results under the same protocol. We operationalize the graph structural reasoning task as a binary link-inference problem. Deviating from generic data-splitting heuristics, we implement a rigorous partitioning strategy where the training, validation, and test sets are strictly disjoint and non-overlapping, ensuring no leakage of structural priors. Each dataset is standardized to $N{=}15$ nodes (e.g., \text{VN\_15}, \text{FW\_15}, and so on) and graphs are split into train/validation/test sets with samples 8000/2000/2000. Unless otherwise specified, we adopt the same observation settings as the benchmark with state dimensions $d{=}4$ for Springs (\text{SP}) and $d{=}1$ for Netsims (\text{NS}). Following the benchmark, we report AUROC for structure inference. To evaluate uncertainty quality, we also report average posterior entropy (mean Shannon entropy over edges) and ECE, computed on softmax probabilities with 10 equal-width bins. All results are averaged over five random seeds and evaluation metrics can be found in Appendix~\ref{app:met}.

Our approach is a plug-in prior for NRI-family neural structure inference models. We use \text{NRI} \cite{kipf2018neural}, \text{ACD} \cite{lowe2022amortized}, and \text{MPM} \cite{chen2021neural} as backbones, and the only modification is the prior module: we replace the default
independent edge prior with our diffusion-based prior, while keeping the
backbone architecture, training schedule, and all other hyperparameters
unchanged. For binary link inference, ``Uniform'' denotes an uninformative
Bernoulli prior with edge probability $p=0.5$, whereas ``Fixed'' denotes the
sparse fixed prior used in the original NRI implementation, with edge
probability $p=0.03$.
We do not introduce additional non-neural baselines, as our method is designed to be a drop-in component within the same neural inference pipeline. For training, we adopt the benchmark configuration: Adam optimizer with learning rate $5\times 10^{-4}$, batch size $64$, and $800$ epochs, with the learning rate halved every $200$ epochs. The diffusion denoiser is implemented as a multi-layer MLP with hidden dimensions $[128,128]$; other diffusion hyperparameters will be provided in Appendix~\ref{app4}. For robustness tests, we construct perturbed variants of the benchmark datasets: (i) \text{Noise}: We inject additive zero-mean Gaussian observation noise, i.e., $\epsilon\sim\mathcal{N}(0,\sigma^2)$ with $\sigma^2=0.1$; (ii) \text{Short-T}: we truncate the last $10\%$ timesteps of each trajectory; and (iii) \text{Missing}: at each timestep, each variable is set to zero with probability $0.1$. Experiments are conducted on three NVIDIA A6000 GPUs. Our prior does not significantly increase runtime, and runtime statistics are reported in Appendix~\ref{app:runtime}.

\begin{figure*}[h]
  \centering
  \includegraphics[height=5.5cm]{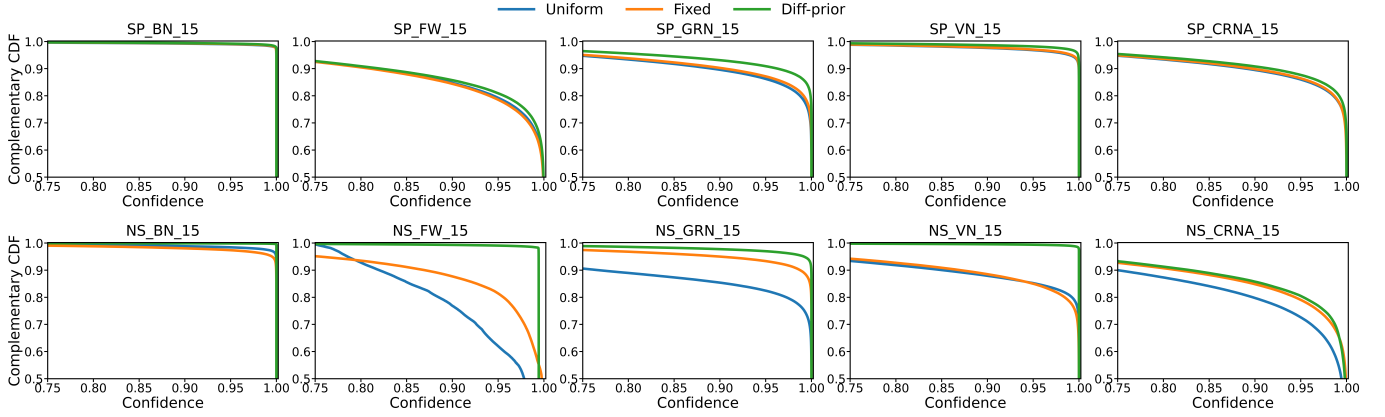}
\caption{Confidence distributions of edge posteriors under the {NRI} backbone across five network families and two dynamics. Each subplot shows the complementary CDF ($1-\mathrm{CDF}$) of confidence $\max(p,1-p)$; higher curves near 1 indicate a larger fraction of high-confidence predictions.}
  \label{fig:3}
\end{figure*}
\subsection{Main Results}
\label{subsec:main-results}
We evaluated the Diff-prior algorithm on the StructInfer benchmark set by applying it to multiple NRI series backbone networks (NRI, ACD, and MPM) without changing any backbone network architecture or training configuration (Table~\ref{tab:table1}).
In simulation data from both dynamics, Diff-prior improved network structure generation compared to uniform or fixed priors. Specifically, in SP dynamics, we observed stable performance improvements across all tested network structures, with average AUROC percentage values increasing by 0.47 (NRI), 0.97 (ACD), and 1.26 (MPM), respectively. In NS dynamics, the advantage of Diff-prior was significantly enhanced due to the more complex dynamic process, which made the posterior marginal probabilities more likely to become close to a uniform distribution: average AUROC values increased by 1.98 (NRI), 4.09 (ACD), and 6.43 (MPM), respectively. The improvement is particularly significant on more challenging network datasets such as FW\_15, GRN\_15, and CRNA\_15, indicating that the Diff-prior performs better on indecisive edge posteriors probabilities where uniform priors tend to produce indecisive ones.

Table~\ref{tab:table1} reveals a clear pattern in the test results across different dynamics: when the dynamic mechanism is more complex (NS compared to SP), the uniform prior is more likely to lead to a near-uniform edge posteriors, thus limiting the backbone's graph structure generation capability; while Diff-prior significantly improves this bottleneck through adaptive prior calibration.
More specifically, the significant improvement brought by Diff-prior on NS indicates that even in models with better expressive power, uniform prior can still be a major source of error, and our Diff-prior can further improve the performance of models with strong expressive power.
Meanwhile, Diff-prior also shows consistent improvement on SP, indicating that it is not only effective under extremely difficult settings, but provides a more reasonable prior form for various dynamics.
Since our modifications are strictly limited to the prior module, the above improvements can be attributed to the structural relevance and calibration capabilities provided by Diff-prior, rather than the benefits from model capacity, training budget, or hyperparameter tuning.
Therefore, Table~\ref{tab:table1} demonstrates that Diff-prior, as a general, backbone-decoupled structural prior plugin, can stabilize the performance of network structure inference models across various inference frameworks and network distributions.

\subsection{Beyond Binary Synthetic Link Inference}
We also evaluate Diff-prior beyond the standard binary synthetic
StructInfer setting. We first consider a multi-relational setting with
$K=3$ edge types and relation strengths of $0$, $0.5$, and $1$. As shown
in Table~\ref{tab:beyond_binary_real}, Diff-prior improves both Micro-F1
and Macro-F1 over the NRI backbone. This suggests that the learned prior
also supports non-binary relation calibration.
We then evaluate Diff-prior on IRMA, a real-world dataset. Despite the
small scale of IRMA ($N=5$, sequence length $7$), Diff-prior slightly
outperforms NRI, suggesting that prior calibration can transfer beyond
synthetic benchmarks.

\begin{table}[t]
\centering
\caption{Additional evaluation beyond binary synthetic link inference.}
\label{tab:beyond_binary_real}
\resizebox{\columnwidth}{!}{
\begin{tabular}{llcccc}
\hline\hline
Setting & Data & Metric & NRI & Diff-prior & Gain \\
\hline
\multirow{2}{*}{Multi-relational} 
& Synthetic, $K=3$ & Micro-F1 & 0.31 & \textbf{0.38} & +0.07 \\
& Synthetic, $K=3$ & Macro-F1 & 0.33 & \textbf{0.38} & +0.05 \\
\midrule
Real-world 
& IRMA, $N=5$ & AUROC & 59.88 & \textbf{60.24} & +0.36 \\
\hline\hline
\end{tabular}
}
\end{table}

\subsection{Posterior Diagnostics}

\label{subsec:diagnostics}

Beyond AUROC, we analyze the posterior distribution and edge confidence
induced by different priors to test whether the independent priors used in
NRI-family models (Uniform/Fixed) lead to more ambiguous structural
estimates. Fig.~\ref{fig:3} plots the confidence distribution for five
network types and two dynamics under the NRI backbone. The horizontal axis
gives the edge confidence $\max(p,1-p)$, and the vertical axis gives
$1-\mathrm{CDF}$, the fraction of edges whose confidence exceeds a given
threshold. Diff-prior is consistently shifted to the right and retains a
higher $1-\mathrm{CDF}$ in the high-confidence region, meaning that more
edges receive probabilities close to 0 or 1. By contrast, the Uniform/Fixed
curves drop earlier, indicating that independent priors leave more edges
with ambiguous posterior beliefs.

Higher confidence alone does not guarantee reliable inference, since
overconfident predictions can still be poorly calibrated. We therefore
measure posterior quality using two complementary metrics: the average
Shannon entropy of edge posteriors, where lower values indicate more
concentrated beliefs, and ECE with 10 confidence bins, where lower values
indicate better calibration. As shown in Table~\ref{tab:table2}, Diff-prior
reduces both entropy and ECE in most network/dynamics combinations. On
NS\_VN, entropy decreases from 0.0281/0.0268 to 0.0011 and ECE from
0.0226/0.0189 to 0.0007; on NS\_GRN, entropy decreases from 0.0374 under
Uniform to 0.0052 and ECE from 0.0296 to 0.0041. The same pattern appears
under SP dynamics: on SP\_GRN, entropy decreases from 0.0219/0.0222 to
0.0168 and ECE from 0.0209/0.0189 to 0.0121; on SP\_FW, entropy decreases
from 0.0380/0.0361 to 0.0337 and ECE from 0.0241/0.0253 to 0.0239.

To check whether this lower entropy reflects overconfident sharpening, we
evaluate calibration on the 10\% lowest-entropy test samples, which
correspond to the most confident posterior predictions.
Table~\ref{tab:low_entropy_calibration} reports ECE and NLL under the NRI
backbone. In this low-entropy regime, Diff-prior reduces average ECE from
0.030 to 0.022 and average NLL from 4.691 to 4.378. Although the improvement
is not uniform across every dataset, the overall trend suggests that
Diff-prior remains more reliable even among highly confident predictions.
Together, Fig.~\ref{fig:3}, Table~\ref{tab:table2}, and
Table~\ref{tab:low_entropy_calibration} support the view of Diff-prior as a
posterior calibration mechanism: it turns diffuse edge beliefs into more
decisive posteriors while preserving, and often improving, calibration
quality.

\begin{table}[t]
\centering
\setlength{\tabcolsep}{4pt} 
\caption{Posterior diagnostics and calibration under \text{NRI}: average entropy and ECE (10 bins) of edge posteriors. \textbf{Bold} indicates the best (lowest) value.}
\label{tab:table2}
\renewcommand{\arraystretch}{0.95} 
\begin{adjustbox}{max width=\linewidth}
\begin{tabular}{c|ccc|ccc}
\hline\hline
\multirow{2}{*}{Networks} & \multicolumn{3}{c|}{Entropy}   & \multicolumn{3}{c}{ECE}  \\ \cline{2-7} 
  & Uniform & Fixed  & Diff-prior & Uniform & Fixed   & Diff-prior   \\ \hline
SP\_VN  & 0.0055  & 0.0055 & \textbf{0.0029}   & 0.0033  & 0.0041  & \textbf{0.0023}  \\
SP\_FW  & 0.038   & 0.0361 & \textbf{0.0337}  & 0.0241  & 0.0253 & \textbf{0.0239} \\
SP\_BN  & 0.0018  & 0.0015 & \textbf{0.0013}   & 0.0013  & 0.0008  & \textbf{0.0007}  \\
SP\_GRN   & 0.0219  & 0.0222 & \textbf{0.0168}   & 0.0209  & 0.0189  & \textbf{0.0121}  \\
SP\_CRNA  & 0.0234  & 0.0246 & \textbf{0.0222}   & 0.0204  & 0.0172  & \textbf{0.0153}  \\
NS\_VN  & 0.0281  & 0.0268 & \textbf{0.0011}   & 0.0226  & 0.0189  & \textbf{0.0007}  \\
NS\_FW  & 0.0274  & 0.0213 & \textbf{0.0025}   & 0.0361  & 0.0036  & \textbf{0.0027}  \\
NS\_BN  & 0.0024  & 0.0043 & \textbf{0.0003}   & 0.0023  & 0.0047  & \textbf{0.0002}  \\
NS\_GRN   & 0.0374  & 0.0117 & \textbf{0.0052}   & 0.0296  & 0.0105  & \textbf{0.0041}  \\
NS\_CRNA  & 0.0337  & 0.0461 & \textbf{0.0303}   & 0.0239  & 0.0241  & \textbf{0.0161}  \\
\hline\hline
\end{tabular}
\end{adjustbox}
\end{table}
\begin{table}[t]
\centering
\caption{Calibration on the 10\% lowest-entropy test samples under the NRI backbone. Lower ECE and NLL indicate better calibrated and more reliable confident predictions.}
\label{tab:low_entropy_calibration}
\resizebox{\columnwidth}{!}{
\begin{tabular}{llcccc}
\hline\hline
Network & ECE (Uni.) & ECE (Diff.) & NLL (Uni.) & NLL (Diff.) \\
\hline
SP\_FW   & 0.177 & \textbf{0.171} & 7.637 & \textbf{7.610} \\
SP\_GRN  & 0.012 & \textbf{0.005} & 3.750 & \textbf{3.040} \\
SP\_CRNA & 0.012 & \textbf{0.009} & 5.380 & \textbf{4.470} \\
SP\_VN   & 0.001 & \textbf{0.000} & 1.620 & \textbf{1.580} \\
SP\_BN   & \textbf{0.000} & \textbf{0.000} & \textbf{0.320} & 0.420 \\
\midrule
NS\_FW   & \textbf{0.000} & \textbf{0.000} & \textbf{14.630} & NS\_14.690 \\
NS\_GRN  & 0.014 & \textbf{0.001} & 5.430 & \textbf{4.460} \\
NS\_CRNA & \textbf{0.000} & \textbf{0.000} & 7.180 & \textbf{6.570} \\
NS\_VN   & 0.080 & \textbf{0.035} & 0.620 & \textbf{0.460} \\
NS\_BN   & \textbf{0.000} & \textbf{0.000} & \textbf{0.340} & 0.480 \\
\hline\hline
\end{tabular}
}
\end{table}
\subsection{Higher-Order Structural Pattern Recovery}

The previous diagnostics focus on edge-level posterior confidence and
calibration. We next examine whether Diff-prior also improves graph-level
structural recovery. We compare the inferred graphs with the ground-truth
graphs using higher-order structural statistics: density ($D$), clustering
coefficient ($C$), modularity ($M$), average degree ($AD$), and triads
($T$). For each statistic, we report the absolute deviation between the
inferred and ground-truth graphs. Smaller values indicate better recovery of
the corresponding structural pattern.

Table~\ref{tab:motif_recovery} shows that Diff-prior reduces the deviation
in most metric-network pairs. The gains are clearest on GRN-NS and VN-NS,
where Diff-prior improves all reported statistics. On CRNA-NS, it improves
density, clustering coefficient, modularity, and average degree, although
the triad deviation increases. FW-NS shows a more mixed pattern: Diff-prior
improves clustering coefficient and modularity but not every statistic.
These results indicate that the learned diffusion prior captures
graph-level regularities beyond independent edge confidence, although
higher-order motif recovery remains dataset-dependent.

\begin{table}[t]
\centering
\caption{Higher-order structural pattern recovery under the NRI backbone.
Each entry reports the absolute deviation from the ground-truth graph as
Uniform/Diff-prior; lower is better.}
\label{tab:motif_recovery}
\resizebox{\columnwidth}{!}{
\begin{tabular}{lcccc}
\hline
Metric & CRNA-NS & FW-NS & GRN-NS & VN-NS \\
\midrule
$\Delta D$  & 0.32/\textbf{0.21} & \textbf{0.23}/0.32 & 0.11/\textbf{0.05} & 0.15/\textbf{0.01} \\
$\Delta C$  & 0.39/\textbf{0.18} & 0.28/\textbf{0.27} & 0.06/\textbf{0.00} & 0.44/\textbf{0.12} \\
$\Delta M$  & 0.24/\textbf{0.23} & 0.06/\textbf{0.05} & 0.03/\textbf{0.02} & 0.28/\textbf{0.13} \\
$\Delta AD$ & 4.53/\textbf{2.91} & \textbf{3.26}/4.41 & 1.60/\textbf{0.67} & 2.07/\textbf{0.15} \\
$\Delta T$  & \textbf{18.00}/170.68 & \textbf{29.09}/58.52 & 36.23/\textbf{24.43} & 22.30/\textbf{1.50} \\
\hline
\end{tabular}
}
\end{table}

\subsection{Robustness tests}
To verify whether Diff-prior is more effective under more challenging conditions, we constructed three types of robustness tests on CRNA: Noise (additive noise with $\epsilon\sim\mathcal{N}(0,\sigma^2)$ with $\sigma^2=0.1$), Short-T (truncating the last 10\% of time steps), and Missing (setting the time to zero with a 10\% probability at each step), and evaluated them under $N=15$ and $N=30$ respectively (Table~\ref{tab:table3}).
These perturbations systematically weaken the trajectory information, making the uniform prior model more prone to generating a confused posterior distribution, thus better testing the effect of the relevant prior.
The results show that Diff-prior achieves a positive improvement over Uniform in Entropy, ECE reduction, and AUROC increase under all robustness tests.
Specifically, taking NRI\_CRNA\_15 as an example, under Noise, Diff-prior improves Entropy from 0.0274 to 0.0265, ECE from 0.0234 to 0.0178, and AUROC from 54.33 to 55.41. This demonstrates the noise immunity of diff-prior.
In Short-T, Entropy improves from 0.0282 to 0.0271, ECE from 0.0152 to 0.0128, and AUROC from 54.56 to 55.06, indicating that Diff-prior still achieves improved performance even with shorter trajectories.
In the Missing test, Entropy improved from 0.0481 to 0.0432, ECE from 0.0358 to 0.0341, and AUROC from 58.58 to 59.44, indicating that the diff-prior is also more robust in missing scenarios.
Consistent conclusions were observed for NRI\_CRNA\_30 dataset with nodes num increase to 30. Robustness testing further demonstrates the effectiveness of our proposed Diff-prior. The Diff-prior can alleviate the near-uniform posterior bottleneck when information is insufficient and achieves improved performance in complex scenarios.

\subsection{Ablation studies}
\label{subsec:ablations}
{\small
\begin{table}[t]
\setlength{\tabcolsep}{2pt}
\caption{Robustness tests on \text{CRNA} under \text{NRI}: \text{Noise}, \text{Short-T}, and \text{Missing} for $N{=}15/30$. We report entropy, ECE, and AUROC. \textbf{Bold} indicates the best result.}
\label{tab:table3}
\begin{tabular}{cccccccc}
\hline
\hline
\multicolumn{2}{c|}{\multirow{2}{*}{Prior}}  & \multicolumn{3}{c|}{NRI\_CRNA\_15}   & \multicolumn{3}{c}{NRI\_CRNA\_30}  \\ \cline{3-8} 
\multicolumn{2}{c|}{}  & Entropy  & ECE  & \multicolumn{1}{c|}{AUROC}   & Entropy  & ECE  & AUROC  \\ \hline
\multicolumn{1}{c|}{\multirow{2}{*}{Noise}}   & \multicolumn{1}{c|}{Uniform} & 0.0274   & 0.0234   & \multicolumn{1}{c|}{54.33} & 0.1010   & 0.3668   &  59.36  \\
\multicolumn{1}{c|}{}   & \multicolumn{1}{c|}{Diff-prior}  & \textbf{0.0265}   & \textbf{0.0178}   & \multicolumn{1}{c|}{\textbf{55.41}} & \textbf{0.0988}   & \textbf{0.3642}   & \textbf{59.53}   \\ \hline
\multicolumn{1}{c|}{\multirow{2}{*}{Short-T}} & \multicolumn{1}{c|}{Uniform} & 0.0282   & 0.0152   & \multicolumn{1}{c|}{54.56} & 0.1135   & 0.4383  & 50.36  \\
\multicolumn{1}{c|}{}   & \multicolumn{1}{c|}{Diff-prior}  & \textbf{0.0271}   & \textbf{0.0128}   & \multicolumn{1}{c|}{\textbf{55.06}} & \textbf{0.1034}   & \textbf{0.4322}   & \textbf{50.45}   \\ \hline
\multicolumn{1}{c|}{\multirow{2}{*}{Missing}} & \multicolumn{1}{c|}{Uniform} & 0.0481   & 0.0358   & \multicolumn{1}{c|}{58.58} & 0.1472   & 0.3889   & 53.988   \\
\multicolumn{1}{c|}{}   & \multicolumn{1}{c|}{Diff-prior}  & \textbf{0.0432}   & \textbf{0.0341}   & \multicolumn{1}{c|}{\textbf{59.44}} & \textbf{0.1429}   & \textbf{0.3839}   & \textbf{53.989}   \\
\hline
\hline
\end{tabular}
\end{table}}

{\small
\begin{table}[t]
\setlength{\tabcolsep}{2pt}
\caption{Ablation and controls on \text{FW\_NRI\_SP} and \text{GRN\_NRI\_SP}: removing $\gamma$ and replacing Diff-prior with a parameter-matched MLP prior. \textbf{Bold} is the best result.}
\label{tab:table4}
\begin{tabular}{lccccccl}
\hline
\hline
\multicolumn{1}{c|}{\multirow{2}{*}{Prior Variant}}   & \multicolumn{3}{c}{FW\_NRI\_SP}  & \multicolumn{3}{c}{GRN\_NRI\_SP}   &  \\ \cline{2-7}
\multicolumn{1}{l|}{}  & Entropy  & ECE  & AUROC  & Entropy  & ECE  & AUROC  &  \\ \cline{1-7}
\multicolumn{1}{c|}{Uniform}   & 0.038  & 0.0241   & 81.09  & 0.0219   & 0.0209   & 89.94  &  \\
\multicolumn{1}{c|}{Fixed}   & 0.0361   & 0.0253   & 81.01  & 0.0222   & 0.0189   & 90.22  &  \\
\multicolumn{1}{c|}{Diff-prior}  & \textbf{0.0337}  & \textbf{0.0239}  & \textbf{81.51}   & \textbf{0.0168}  & \textbf{0.0121}  & \textbf{90.58}   &  \\
\multicolumn{1}{c|}{Diff-prior w/o $\gamma$} & 0.0763   & 0.2173   & 80.91  & 0.0368   & 0.1641   & 87.53  &  \\
\multicolumn{1}{c|}{PM-MLP Prior}   & 0.113  & 0.1975   & 80.84  & 0.0554   & 0.1008   & 87.49  &  \\
   \hline
   \hline
\end{tabular}
\end{table}}
We use ablations on FW\_NRI\_SP and GRN\_NRI\_SP to identify where the gains
of Diff-prior come from (Table~\ref{tab:table4}). The full model obtains
Entropy/ECE/AUROC values of 0.0337/0.0239/81.51 on FW and
0.0168/0.0121/90.58 on GRN.

Removing the residual calibration intensity $\gamma$ (Diff-prior w/o
$\gamma$) sharply worsens calibration and also lowers AUROC. On FW, ECE
increases from 0.0239 to 0.2173, while AUROC drops from 81.51 to 80.91. On
GRN, ECE increases from 0.0121 to 0.1641, while AUROC drops from 90.58 to
87.53. Thus, the $\gamma$-controlled calibration step is essential to the
diffusion prior.

We also replace the Diff-prior module (PM-MLP prior) with a
parameter-matched MLP to control for model size. This variant again performs
worse: ECE/AUROC becomes 0.1975/80.84 on FW and 0.1008/87.49 on GRN. These
results show that the improvement does not come from extra parameters or an
arbitrary adaptive prior. It comes from the diffusion-based prior and its
calibration mechanism.

\section{Discussion}
Our experiments consistently show that in NRI-family models, commonly used independent priors (Uniform/Fixed) cause some edge posteriors to remain in a near-moderate region, thus limiting structural recovery performance and posterior reliability.
Diff-prior constructs an adaptive prior distribution through diffusion denoising, which can be understood as a learnable posterior calibration mechanism, so that the model can obtain a more deterministic and accurate edge posteriors distribution without changing the backbone and training configuration.
The main results in Table~\ref{tab:table1} are consistent with this mechanism: Diff-prior delivers improvements on all three backbones (NRI/ACD/MPM) and two types of dynamics (SP/NS), with the gain being more significant on NS.
Posterior diagnostics (Fig.~\ref{fig:3} and Table~\ref{tab:table2}) provide direct evidence: the confidence distribution of Diff-prior shows more edges falling in the high-confidence interval, while Entropy and ECE decrease synchronously in most settings, indicating that its improvement is not simply due to overconfidence.
Robustness testing (Table~\ref{tab:table3}) further validates robustness: under noise, short sequences, and missing observation perturbations, Diff-prior maintains a consistent trend of decreasing Entropy and ECE/increasing AUROC, showing the same conclusions at both N=15 and N=30 scales.

As shown in Table~\ref{tab:table4}, the ablation study further confirmed the above conclusions. When we removed the key calibration intensity $\gamma$, the ECE deteriorated significantly and the AUROC decreased, indicating that this calibration step is crucial to the effectiveness of Diff-prior.
We replaced the Diff prior with a parameter-matched MLP prior, but it failed to reproduce the performance and calibration gain, indicating that the improvement did not come from additional parameters or arbitrary adaptive priors, instead, it relied on prior enhancements caused by diffusion denoising.
We have systematically validated this on multiple network families, two types of dynamics, and two scales ($N=15/30$), current experiments still mainly focus on the StructInfer baseline setting and binary edge inference task. Further expansion to a wider range of data and relational forms is a meaningful direction.
Our diagnostics mainly revolve around AUROC, entropy, and ECE; future work could supplement this with more fine-grained reliability analysis to more comprehensively characterize posterior quality.
This work demonstrates that using diffusion models to construct pluggable relevant structural priors can simultaneously improve structure discovery performance and posterior reliability across various NRI-family inference frameworks.

\section{Acknowledgments}
This work is supported by the National Key R{\&}D Program of China under Grant No. 2022ZD0120004, the Zhishan Youth Scholar Program, the National Natural Science Foundation of China under Grant Nos. 62233004, 62273090, and the Jiangsu Provincial Scientific Research Center of Applied Mathematics under Grant No. BK20233002.

\newpage

\bibliographystyle{unsrt}
\bibliography{sample-base}

%

\appendix

\section{Why Use Joint Distribution Optimization}
\label{app:latent_score_matching}

This appendix explains why simplifications that drop terms treated as constants (e.g., $C_1,C_2$)
in ESM--DSM style derivations can be invalid when the encoder is trained jointly with a latent diffusion prior.
The discussion follows the key observation highlighted by Graphusion for latent diffusion models \cite{yang2024graphusion}.

\subsection{ESM--DSM connection and the origin of constant terms}
\label{app:esm_dsm_origin}

Let $q(Z_t\mid Z_0)$ be a forward noising process and define the induced marginal
\begin{equation}
q_{\theta_{\mathrm{enc}}}(Z_t)
=
\int q(Z_t\mid Z_0)\, q_{\theta_{\mathrm{enc}}}(Z_0)\, dZ_0,
\label{eq:appA_qt_marginal}
\end{equation}
where $q_{\theta_{\mathrm{enc}}}(Z_0)$ denotes the (possibly learned) latent data distribution induced by the encoder
(e.g., an aggregated encoder posterior).
Consider explicit score matching (ESM) at noise level $t$:
\begin{equation}
\mathcal{J}_{\mathrm{ESM}}(\theta_{\mathrm{diff}};\theta_{\mathrm{enc}})
=
\mathbb{E}_{t}\;
\mathbb{E}_{Z_t\sim q_{\theta_{\mathrm{enc}}}(Z_t)}
\Bigl\|
s_{\theta_{\mathrm{diff}}}(Z_t,t) - \nabla_{Z_t}\log q_{\theta_{\mathrm{enc}}}(Z_t)
\Bigr\|_2^2,
\label{eq:appA_esm}
\end{equation}
where $s_{\theta_{\mathrm{diff}}}(Z_t,t)$ is the learned score network.

Expanding Eq.~\eqref{eq:appA_esm} gives
\begin{align}
\mathcal{J}_{\mathrm{ESM}}(\theta_{\mathrm{diff}};\theta_{\mathrm{enc}})
&=
\mathbb{E}_{t}\;
\mathbb{E}_{q_{\theta_{\mathrm{enc}}}(Z_t)}
\bigl\| s_{\theta_{\mathrm{diff}}}(Z_t,t)\bigr\|_2^2
\nonumber\\
&\quad
-2\,\mathbb{E}_{t}\;
\mathbb{E}_{q_{\theta_{\mathrm{enc}}}(Z_t)}
\Bigl\langle s_{\theta_{\mathrm{diff}}}(Z_t,t), \nabla_{Z_t}\log q_{\theta_{\mathrm{enc}}}(Z_t)\Bigr\rangle
\\
&\quad+ C_1(\theta_{\mathrm{enc}}),
\label{eq:appA_esm_expand}
\end{align}
where
\begin{equation}
C_1(\theta_{\mathrm{enc}})
:=
\mathbb{E}_{t}\;
\mathbb{E}_{q_{\theta_{\mathrm{enc}}}(Z_t)}
\bigl\|\nabla_{Z_t}\log q_{\theta_{\mathrm{enc}}}(Z_t)\bigr\|_2^2.
\label{eq:appA_c1_def}
\end{equation}
The term $C_1(\theta_{\mathrm{enc}})$ does not depend on $\theta_{\mathrm{diff}}$, but it generally depends on
the marginal distribution $q_{\theta_{\mathrm{enc}}}(Z_t)$.
Under standard regularity conditions, integration-by-parts arguments can further eliminate
$\nabla\log q_{\theta_{\mathrm{enc}}}(Z_t)$ from the cross term, leading to implicit score matching and denoising variants.
In such equivalences, additional terms (often grouped as $C_2$) may appear and are sometimes informally described as constants.

\subsection{Learned-latent diffusion makes the marginal parameter-dependent}
\label{app:latent_marginal_dep}

In learned latent diffusion, $q_{\theta_{\mathrm{enc}}}(Z_0)$ is induced by the encoder and thus varies with
$\theta_{\mathrm{enc}}$. Consequently, the marginal $q_{\theta_{\mathrm{enc}}}(Z_t)$ in Eq.~\eqref{eq:appA_qt_marginal}
also varies with $\theta_{\mathrm{enc}}$ through the marginalization integral. This dependence is exactly what makes
the marginal score $\nabla_{Z_t}\log q_{\theta_{\mathrm{enc}}}(Z_t)$ difficult to access analytically in latent space.
Graphusion discusses this issue explicitly in latent diffusion settings \cite{yang2024graphusion}.

\subsection{Why $C_1,C_2$ are not constant under joint optimization}
\label{app:constants_not_constant}

If one optimizes only the diffusion/score parameters $\theta_{\mathrm{diff}}$ while treating
$q_{\theta_{\mathrm{enc}}}(Z_0)$ as fixed (e.g., a strictly two-stage pipeline where the encoder is frozen),
then $C_1(\theta_{\mathrm{enc}})$ and related terms can indeed be treated as constants {with respect to}
$\theta_{\mathrm{diff}}$ and dropped.

However, under joint training of $(\theta_{\mathrm{diff}},\theta_{\mathrm{enc}})$, these terms generally contribute
non-trivially to the gradient with respect to $\theta_{\mathrm{enc}}$. In particular, differentiating
Eq.~\eqref{eq:appA_esm_expand} yields the decomposition
{\small
\begin{align}
&\nabla_{\theta_{\mathrm{enc}}}\mathcal{J}_{\mathrm{ESM}}(\theta_{\mathrm{diff}};\theta_{\mathrm{enc}})
=
\nabla_{\theta_{\mathrm{enc}}}
\Bigg(
\mathbb{E}_{t}\;
\mathbb{E}_{q_{\theta_{\mathrm{enc}}}(Z_t)}
\bigl\| s_{\theta_{\mathrm{diff}}}(Z_t,t)\bigr\|_2^2
\Bigg)
\nonumber\\
&\quad
-2\,\nabla_{\theta_{\mathrm{enc}}}
\Bigg(
\mathbb{E}_{t}\;
\mathbb{E}_{q_{\theta_{\mathrm{enc}}}(Z_t)}
\Bigl\langle s_{\theta_{\mathrm{diff}}}(Z_t,t), \nabla_{Z_t}\log q_{\theta_{\mathrm{enc}}}(Z_t)\Bigr\rangle
\Bigg)
\nonumber
+\nabla_{\theta_{\mathrm{enc}}} C_1(\theta_{\mathrm{enc}}),
\label{eq:appA_grad_decomp}
\end{align}}
and $\nabla_{\theta_{\mathrm{enc}}} C_1(\theta_{\mathrm{enc}})$ is generally non-zero because
$q_{\theta_{\mathrm{enc}}}(Z_t)$ depends on $\theta_{\mathrm{enc}}$ via Eq.~\eqref{eq:appA_qt_marginal}.
The same conclusion applies to other terms arising in ESM--DSM equivalences (often denoted $C_2$):
their constancy typically holds only when the underlying data distribution is fixed.
Therefore, replacing the intended regularizer by a simplified target that drops such terms changes the
optimization problem and may yield a mismatch in the training signal passed to the encoder.

\subsection{Implication for our objective}
\label{app:implication_our_objective}

To obtain a tractable objective that remains valid under joint training of the encoder, diffusion prior,
and decoder, we avoid relying on constant-term simplifications.
Instead, the main text upper bounds the marginal regularizer
\begin{equation}
\mathrm{KL}\!\left(q_{\theta_{\mathrm{enc}}}(Z_0\mid\mathbf{X})\,\|\,p_{\theta_{\mathrm{diff}}}(Z_0)\right)
\end{equation}
by a joint diffusion-path KL (Thm.~\ref{thm:marginal_joint_kl}),
which admits a diffusion-style decomposition (Sec.~\ref{sec:diff_reg}) and an efficient single-timestep estimator (Sec.~\ref{sec:efficient_diff}).

\section{Proofs for the diffusion KL surrogate}
\label{app:diffusion_kl_proofs}
\subsection{Proof of Theorem~\ref{thm:marginal_joint_kl}}
\label{app:proof_marginal_joint_kl}

For clarity, we prove a slightly more general statement and then apply it to
$Q_{\theta_{\mathrm{enc}}}(Z_{0:S}\mid\mathbf{X})$ and $P_{\theta_{\mathrm{diff}}}(Z_{0:S})$.

\begin{lemma}[Marginal-to-joint KL bound]
\label{lem:marginal_joint_general}
Let $(\Omega,\mathcal{F})$ be a measurable space and let $Q,P$ be probability measures on it.
Let $\pi:\Omega\to\mathcal{Z}$ be a measurable map, and denote the push-forward (marginal) measures
$Q_0 := \pi_{\#}Q$ and $P_0 := \pi_{\#}P$ on $(\mathcal{Z},\mathcal{B}(\mathcal{Z}))$.
Then
\begin{equation}
\mathrm{KL}(Q_0\|P_0)\;\le\;\mathrm{KL}(Q\|P).
\end{equation}
\end{lemma}

\begin{proof}
\textbf{Step 0 (the case $Q\not\ll P$).}
If $Q\not\ll P$, then by the standard convention $\mathrm{KL}(Q\|P)=+\infty$.
Since $\mathrm{KL}(Q_0\|P_0)\ge 0$ (possibly $+\infty$), the inequality holds trivially:
$\mathrm{KL}(Q_0\|P_0)\le \mathrm{KL}(Q\|P)=+\infty$.
Hence, it suffices to consider the case $Q\ll P$.

\textbf{Step 1 (Radon--Nikodym derivative of the joint).}
Assume $Q\ll P$ and let
\begin{equation}
f := \frac{dQ}{dP}.
\end{equation}
Then $f\ge 0$ $P$-a.s.\ and $\mathbb{E}_{P}[f]=\int_{\Omega} f\,dP = Q(\Omega)=1$.

\textbf{Step 2 (Radon--Nikodym derivative of the marginal).}
Let $\mathcal{G}:=\sigma(\pi)$ be the $\sigma$-algebra generated by $\pi$ and define
\begin{equation}
g := \mathbb{E}_{P}[f\mid \mathcal{G}].
\end{equation}
Then $g$ is $\mathcal{G}$-measurable. Hence there exists a measurable function
$\tilde g:\mathcal{Z}\to[0,\infty)$ such that
\begin{equation}
g(\omega)=\tilde g(\pi(\omega))\quad P\text{-a.s.}
\end{equation}

We claim that $Q_0\ll P_0$ and that $\tilde g$ is the Radon--Nikodym derivative of $Q_0$ w.r.t.\ $P_0$, i.e.,
\begin{equation}
\frac{dQ_0}{dP_0}(z)=\tilde g(z)\quad P_0\text{-a.s.}
\label{eq:rn_marginal}
\end{equation}
To show this, take any measurable set $A\subseteq\mathcal{Z}$. By definition of push-forward,
\begin{equation}
\begin{aligned}
Q_0(A)
&= Q(\pi\in A)
= \int_{\Omega}\mathbf{1}\{\pi(\omega)\in A\}\,dQ(\omega) \\
&= \int_{\Omega}\mathbf{1}\{\pi(\omega)\in A\}\,f(\omega)\,dP(\omega).
\end{aligned}
\end{equation}
Since $\mathbf{1}\{\pi(\omega)\in A\}$ is $\mathcal{G}$-measurable, the tower property gives
\begin{equation}
\begin{aligned}
\int_{\Omega}\mathbf{1}\{\pi\in A\}\,f\,dP
&=
\int_{\Omega}\mathbf{1}\{\pi\in A\}\,\mathbb{E}_{P}[f\mid\mathcal{G}]\,dP \\
&=
\int_{\Omega}\mathbf{1}\{\pi\in A\}\,g(\omega)\,dP(\omega).
\end{aligned}
\end{equation}
Using $g(\omega)=\tilde g(\pi(\omega))$ $P$-a.s.\ and the change-of-variables identity for push-forward measures,
$\int_{\Omega}\varphi(\pi(\omega))\,dP(\omega)=\int_{\mathcal{Z}}\varphi(z)\,dP_0(z)$, we obtain
\begin{equation}
\begin{aligned}
Q_0(A)
&=\int_{\Omega}\mathbf{1}\{\pi\in A\}\,\tilde g(\pi(\omega))\,dP(\omega)
=\int_{\mathcal{Z}}\mathbf{1}\{z\in A\}\,\tilde g(z)\,dP_0(z),
\end{aligned}
\end{equation}
which is exactly \eqref{eq:rn_marginal}.

\textbf{Step 3 (conclude via conditional Jensen).}
If $\mathrm{KL}(Q\|P)=\int f\log f\,dP=+\infty$, then the desired inequality holds trivially.
Assume henceforth that $\int f\log f\,dP<+\infty$.

Let $\varphi(u)=u\log u$ for $u\ge 0$, which is convex. Since $g=\mathbb{E}_{P}[f\mid\mathcal{G}]$,
conditional Jensen's inequality yields
\begin{equation}
\varphi(g)
=
\varphi\!\left(\mathbb{E}_{P}[f\mid\mathcal{G}]\right)
\le
\mathbb{E}_{P}[\varphi(f)\mid\mathcal{G}]
\qquad P\text{-a.s.}
\end{equation}
Taking expectation under $P$ gives
\begin{equation}
\mathbb{E}_{P}[g\log g]\le \mathbb{E}_{P}[f\log f].
\end{equation}
By \eqref{eq:rn_marginal} and $g(\omega)=\tilde g(\pi(\omega))$ $P$-a.s., we have
\begin{equation}
\begin{aligned}
\mathrm{KL}(Q_0\|P_0)
&=
\int_{\mathcal{Z}} \tilde g(z)\log \tilde g(z)\,dP_0(z) \\
&=
\int_{\Omega} g(\omega)\log g(\omega)\,dP(\omega)
=
\mathbb{E}_{P}[g\log g].
\end{aligned}
\end{equation}
Similarly,
\begin{equation}
\mathrm{KL}(Q\|P)=\int_{\Omega} f\log f\,dP=\mathbb{E}_{P}[f\log f].
\end{equation}
Therefore,
\begin{equation}
\mathrm{KL}(Q_0\|P_0)\le \mathrm{KL}(Q\|P),
\end{equation}
which proves Lemma~\ref{lem:marginal_joint_general}.
\end{proof}

\paragraph{Applying Lemma~\ref{lem:marginal_joint_general} to Theorem~\ref{thm:marginal_joint_kl}.}
In our setting, let $\Omega$ be the space of diffusion paths $Z_{0:S}$ and let $\pi(Z_{0:S})=Z_0$.
Take $Q=Q_{\theta_{\mathrm{enc}}}(Z_{0:S}\mid\mathbf{X})$ and $P=P_{\theta_{\mathrm{diff}}}(Z_{0:S})$.
Then $Q_0=q_{\theta_{\mathrm{enc}}}(Z_0\mid\mathbf{X})$ and $P_0=p_{\theta_{\mathrm{diff}}}(Z_0)$, and
Lemma~\ref{lem:marginal_joint_general} gives Eq.~\eqref{eq:ldiff_upper}.

\subsection{Proof of the Decomposition of $\widetilde{\mathcal{L}}_{\mathrm{diff}}(\mathbf{X})$}
\label{app:proof_C2}

In this appendix we prove the following statement :

\begin{proposition}[Joint-KL identity]
\label{prop:joint_kl_identity}
Let
\begin{equation}
Q_{\theta_{\mathrm{enc}}}(Z_{0:S}\mid \mathbf{X})
:= q_{\theta_{\mathrm{enc}}}(Z_0\mid \mathbf{X})\, q(Z_{1:S}\mid Z_0),
P_{\theta_{\mathrm{diff}}}(Z_{0:S})
:= p_{\theta_{\mathrm{diff}}}(Z_{0:S}),
\end{equation}
and define the diffusion ELBO for a fixed $Z_0$ as
\begin{equation}
\mathrm{ELBO}_{\theta_{\mathrm{diff}}}(Z_0)
:= \mathbb{E}_{q(Z_{1:S}\mid Z_0)}
\big[\log p_{\theta_{\mathrm{diff}}}(Z_{0:S})-\log q(Z_{1:S}\mid Z_0)\big].
\label{eq:diff_elbo_def_app}
\end{equation}
Then the joint KL admits the following identity:
\begin{align}
\widetilde{\mathcal{L}}_{\mathrm{diff}}(\mathbf{X})
:&= \mathrm{KL}\!\left(Q_{\theta_{\mathrm{enc}}}(\cdot\mid \mathbf{X}) \,\|\, P_{\theta_{\mathrm{diff}}}(\cdot)\right)
\\&=
\mathbb{E}_{q_{\theta_{\mathrm{enc}}}(Z_0\mid \mathbf{X})}
\Big[\log q_{\theta_{\mathrm{enc}}}(Z_0\mid \mathbf{X})
- \mathrm{ELBO}_{\theta_{\mathrm{diff}}}(Z_0)\Big].
\label{eq:joint_kl_identity_app}
\end{align}
\end{proposition}

\begin{proof}
By definition of KL divergence and the Radon--Nikodym derivative,
\begin{align}
\widetilde{\mathcal{L}}_{\mathrm{diff}}(\mathbf{X})
= &\mathbb{E}_{Q_{\theta_{\mathrm{enc}}}(Z_{0:S}\mid \mathbf{X})}
\left[
\log \frac{Q_{\theta_{\mathrm{enc}}}(Z_{0:S}\mid \mathbf{X})}{P_{\theta_{\mathrm{diff}}}(Z_{0:S})}
\right] \nonumber\\
=& \mathbb{E}_{q_{\theta_{\mathrm{enc}}}(Z_0\mid \mathbf{X})\,q(Z_{1:S}\mid Z_0)}
[
\log q_{\theta_{\mathrm{enc}}}(Z_0\mid \mathbf{X}) \\
& +\log q(Z_{1:S}\mid Z_0)
-\log p_{\theta_{\mathrm{diff}}}(Z_{0:S})
].
\label{eq:joint_kl_expand_app}
\end{align}
Now apply the tower property (iterated expectation) with respect to
$Z_0\sim q_{\theta_{\mathrm{enc}}}(Z_0\mid \mathbf{X})$ and
$Z_{1:S}\sim q(Z_{1:S}\mid Z_0)$:
\begin{align}
\widetilde{\mathcal{L}}_{\mathrm{diff}}(\mathbf{X})
=&
\mathbb{E}_{q_{\theta_{\mathrm{enc}}}(Z_0\mid \mathbf{X})}
[
\mathbb{E}_{q(Z_{1:S}\mid Z_0)}
(
\log q_{\theta_{\mathrm{enc}}}(Z_0\mid \mathbf{X})\\
&+\log q(Z_{1:S}\mid Z_0)
-\log p_{\theta_{\mathrm{diff}}}(Z_{0:S})
)
] \nonumber\\
=&
\mathbb{E}_{q_{\theta_{\mathrm{enc}}}(Z_0\mid \mathbf{X})}
[
\log q_{\theta_{\mathrm{enc}}}(Z_0\mid \mathbf{X})\\
&+
\mathbb{E}_{q(Z_{1:S}\mid Z_0)}
(
\log q(Z_{1:S}\mid Z_0)
-\log p_{\theta_{\mathrm{diff}}}(Z_{0:S})
)
] \nonumber\\
=&
\mathbb{E}_{q_{\theta_{\mathrm{enc}}}(Z_0\mid \mathbf{X})}
[
\log q_{\theta_{\mathrm{enc}}}(Z_0\mid \mathbf{X})\\
&-
\mathbb{E}_{q(Z_{1:S}\mid Z_0)}
(
\log p_{\theta_{\mathrm{diff}}}(Z_{0:S})-\log q(Z_{1:S}\mid Z_0)
)
].
\end{align}
The inner expectation is exactly $\mathrm{ELBO}_{\theta_{\mathrm{diff}}}(Z_0)$ defined in
Eq.~\eqref{eq:diff_elbo_def_app}, which yields Eq.~\eqref{eq:joint_kl_identity_app}.
\end{proof}

\subsection{Proof of the Diff-prior ELBO decomposition}
\label{app:proof_C3}

We prove the standard per-step decomposition.
Assume the diffusion joint is
\begin{equation}
p_{\theta_{\mathrm{diff}}}(Z_{0:S})
=
p(Z_S)\prod_{s=1}^{S} p_{\theta_{\mathrm{diff}}}(Z_{s-1}\mid Z_s),
\qquad
p(Z_S)=\mathcal{N}(0,\mathbf{I}),
\label{eq:reverse_chain_app}
\end{equation}
and the forward noising process is Markov:
\begin{equation}
q(Z_{1:S}\mid Z_0)=\prod_{s=1}^{S} q(Z_s\mid Z_{s-1}).
\label{eq:forward_chain_app}
\end{equation}

\begin{theorem}[Per-step KL decomposition of $-\mathrm{ELBO}$]
\label{thm:ddpm_elbo_decomp_app}
For any fixed $Z_0$, the negative ELBO in Eq.~\eqref{eq:diff_elbo_def_app} satisfies
\begin{align}
-\mathrm{ELBO}&_{\theta_{\mathrm{diff}}}(Z_0)
=
\underbrace{\mathrm{KL}\!\left(q(Z_S\mid Z_0)\,\|\,p(Z_S)\right)}_{L_S(Z_0)}\\
&+
\sum_{s=2}^{S}
\underbrace{\mathbb{E}_{q(Z_s\mid Z_0)}
\mathrm{KL}\!\left(q(Z_{s-1}\mid Z_s, Z_0)\,\|\,p_{\theta_{\mathrm{diff}}}(Z_{s-1}\mid Z_s)\right)}_{L_{s-1}(Z_0)}\\
&+
\underbrace{\mathbb{E}_{q(Z_1\mid Z_0)}\big[-\log p_{\theta_{\mathrm{diff}}}(Z_0\mid Z_1)\big]}_{L_0(Z_0)}.
\label{eq:elbo_decomp_app}
\end{align}
\end{theorem}

\begin{proof}
Start from the definition of the ELBO:
\begin{align}
-\mathrm{ELBO}_{\theta_{\mathrm{diff}}}(Z_0)
=
\mathbb{E}_{q(Z_{1:S}\mid Z_0)}&
\Big[
-\log p_{\theta_{\mathrm{diff}}}(Z_{0:S})
+\log q(Z_{1:S}\mid Z_0)
\Big] \nonumber\\
=
\mathbb{E}_{q(Z_{1:S}\mid Z_0)}&
[
-\log p(Z_S)
-\sum_{s=1}^{S}\log p_{\theta_{\mathrm{diff}}}(Z_{s-1}\mid Z_s)\\
&+\sum_{s=1}^{S}\log q(Z_s\mid Z_{s-1})
].
\label{eq:neg_elbo_expand1_app}
\end{align}
Next, for each $s\ge 2$, apply Bayes' rule to the forward Markov chain:
\begin{equation}
q(Z_s\mid Z_{s-1})
=
\frac{q(Z_{s-1}\mid Z_s, Z_0)\,q(Z_s\mid Z_0)}{q(Z_{s-1}\mid Z_0)}.
\label{eq:bayes_forward_app}
\end{equation}
Taking logs and summing over $s=2,\dots,S$ gives a telescoping cancellation:
\begin{align}
\sum_{s=2}^{S}\log q(Z_s\mid Z_{s-1})
&=
\sum_{s=2}^{S}\log q(Z_{s-1}\mid Z_s, Z_0)\\
&+\log q(Z_S\mid Z_0)
-\log q(Z_1\mid Z_0).
\label{eq:telescope_app}
\end{align}
Substitute Eq.~\eqref{eq:telescope_app} into Eq.~\eqref{eq:neg_elbo_expand1_app} and separate the $s=1$ terms:
\begin{align}
-\mathrm{ELBO}_{\theta_{\mathrm{diff}}}(Z_0)
=&
\mathbb{E}_{q(Z_{1:S}\mid Z_0)}
\Big[
\log \frac{q(Z_S\mid Z_0)}{p(Z_S)}\\
&+
\sum_{s=2}^{S}
\log \frac{q(Z_{s-1}\mid Z_s, Z_0)}{p_{\theta_{\mathrm{diff}}}(Z_{s-1}\mid Z_s)}
-\log p_{\theta_{\mathrm{diff}}}(Z_0\mid Z_1)
\Big].
\label{eq:neg_elbo_expand2_app}
\end{align}
We now rewrite each term as a KL divergence:

(i) Terminal term.
Since the first term depends only on $Z_S$ (given $Z_0$),
\begin{align}
\mathbb{E}_{q(Z_{1:S}\mid Z_0)}
\left[\log \frac{q(Z_S\mid Z_0)}{p(Z_S)}\right]
&=
\mathbb{E}_{q(Z_S\mid Z_0)}
\left[\log \frac{q(Z_S\mid Z_0)}{p(Z_S)}\right]
\\&=
\mathrm{KL}\!\left(q(Z_S\mid Z_0)\,\|\,p(Z_S)\right).
\end{align}

(ii) Intermediate terms.
For each $s\ge 2$, the corresponding log-ratio depends on $(Z_{s-1},Z_s)$, and conditioning on $Z_s$
yields
\begin{align}
&\mathbb{E}_{q(Z_{1:S}\mid Z_0)}
\left[
\log \frac{q(Z_{s-1}\mid Z_s, Z_0)}{p_{\theta_{\mathrm{diff}}}(Z_{s-1}\mid Z_s)}
\right]\\
&=
\mathbb{E}_{q(Z_s\mid Z_0)}
\Bigg[
\mathbb{E}_{q(Z_{s-1}\mid Z_s, Z_0)}
\left[
\log \frac{q(Z_{s-1}\mid Z_s, Z_0)}{p_{\theta_{\mathrm{diff}}}(Z_{s-1}\mid Z_s)}
\right]
\Bigg] \nonumber\\
&=
\mathbb{E}_{q(Z_s\mid Z_0)}
\mathrm{KL}\!\left(q(Z_{s-1}\mid Z_s, Z_0)\,\|\,p_{\theta_{\mathrm{diff}}}(Z_{s-1}\mid Z_s)\right).
\end{align}

(iii) Reconstruction term.
The remaining term is
\begin{align}
\mathbb{E}_{q(Z_{1:S}\mid Z_0)}\big[-\log p_{\theta_{\mathrm{diff}}}(Z_0\mid Z_1)\big]
=
\mathbb{E}_{q(Z_1\mid Z_0)}\big[-\log p_{\theta_{\mathrm{diff}}}(Z_0\mid Z_1)\big].
\end{align}

Combining (i)--(iii) with Eq.~\eqref{eq:neg_elbo_expand2_app} yields the decomposition
in Eq.~\eqref{eq:elbo_decomp_app}.
\end{proof}
\subsection{ From per-step KL to the $\epsilon$-prediction objective}
\label{app:proof_C4}

This appendix shows that, under the standard Gaussian diffusion parameterization with fixed variances,
the per-step KL terms in Theorem~\ref{thm:ddpm_elbo_decomp_app} reduce (up to known weights and constants)
to a weighted noise-prediction regression loss, yielding Eq.~\eqref{eq:diff_train_loss}.
Our derivation follows the standard DDPM analysis. 

\paragraph{Forward process and closed-form marginals.}
Fix a variance schedule $\{\beta_s\}_{s=1}^S$ with $\alpha_s:=1-\beta_s$ and $\bar\alpha_s:=\prod_{r=1}^s \alpha_r$.
The forward Markov chain is
\begin{equation}
q(Z_s \mid Z_{s-1})=\mathcal{N}\!\left(\sqrt{\alpha_s}\,Z_{s-1},\; \beta_s \mathbf{I}\right).
\label{eq:appA4_forward}
\end{equation}
It admits the closed-form marginal
{\small
\begin{equation}
q(Z_s \mid Z_0)=\mathcal{N}\!\left(\sqrt{\bar\alpha_s}\,Z_0,\; (1-\bar\alpha_s)\mathbf{I}\right),
Z_s=\sqrt{\bar\alpha_s}\,Z_0+\sqrt{1-\bar\alpha_s}\,\epsilon,\;\epsilon\sim\mathcal{N}(0,\mathbf{I}),
\label{eq:appA4_marginal_reparam}
\end{equation}
}

which matches Eq.~\eqref{eq:zt_reparam} in the main text.

\paragraph{Forward posterior.}
The posterior of the forward process is also Gaussian:
\begin{equation}
q(Z_{s-1}\mid Z_s, Z_0)=\mathcal{N}\!\left(\tilde{\mu}_s(Z_s,Z_0),\;\tilde{\beta}_s \mathbf{I}\right),
\label{eq:appA4_forward_posterior}
\end{equation}
with
\begin{align}
\tilde{\beta}_s
&:=\frac{1-\bar\alpha_{s-1}}{1-\bar\alpha_s}\,\beta_s, \label{eq:appA4_beta_tilde}\\
\tilde{\mu}_s(Z_s,Z_0)
&:=\frac{\sqrt{\bar\alpha_{s-1}}\,\beta_s}{1-\bar\alpha_s}\,Z_0
+\frac{\sqrt{\alpha_s}\,(1-\bar\alpha_{s-1})}{1-\bar\alpha_s}\,Z_s.
\label{eq:appA4_mu_tilde}
\end{align}
These are standard identities obtained by Gaussian conditioning.

\paragraph{Reverse model and $\epsilon$-parameterization.}
Assume the learned reverse transition has fixed variance $\sigma_s^2 \mathbf{I}$:
\begin{equation}
p_{\theta_{\mathrm{diff}}}(Z_{s-1}\mid Z_s)=
\mathcal{N}\!\left(\mu_{\theta_{\mathrm{diff}}}(Z_s,s),\;\sigma_s^2 \mathbf{I}\right).
\label{eq:appA4_reverse_gaussian}
\end{equation}
Following DDPM, we parameterize the mean via a noise predictor $\epsilon_{\theta_{\mathrm{diff}}}(Z_s,s)$:
\begin{equation}
\mu_{\theta_{\mathrm{diff}}}(Z_s,s)
:=
\frac{1}{\sqrt{\alpha_s}}
\left(
Z_s - \frac{\beta_s}{\sqrt{1-\bar\alpha_s}}\,\epsilon_{\theta_{\mathrm{diff}}}(Z_s,s)
\right).
\label{eq:appA4_mu_eps_param}
\end{equation}
Moreover, using the reparameterization in Eq.~\eqref{eq:appA4_marginal_reparam}, one can rewrite
$\tilde{\mu}_s(Z_s,Z_0)$ in an equivalent ``noise form'':
\begin{equation}
\tilde{\mu}_s(Z_s,Z_0)
=
\frac{1}{\sqrt{\alpha_s}}
\left(
Z_s - \frac{\beta_s}{\sqrt{1-\bar\alpha_s}}\,\epsilon
\right),
\quad
\text{where } Z_s=\sqrt{\bar\alpha_s}Z_0+\sqrt{1-\bar\alpha_s}\epsilon.
\label{eq:appA4_mu_tilde_noise_form}
\end{equation}
(An explicit verification follows by substituting $Z_s$ into Eq.~\eqref{eq:appA4_mu_tilde} and simplifying.)

\paragraph{Key step: per-step KL becomes weighted MSE.}
Consider the intermediate term in Theorem~\ref{thm:ddpm_elbo_decomp_app}:
\begin{equation}
L_{s-1}(Z_0)
=
\mathbb{E}_{q(Z_s\mid Z_0)}
\mathrm{KL}\!\left(
q(Z_{s-1}\mid Z_s, Z_0)\,\|\,p_{\theta_{\mathrm{diff}}}(Z_{s-1}\mid Z_s)
\right).
\label{eq:appA4_L_sminus1_def}
\end{equation}
Since both distributions are Gaussians with isotropic covariances (Eqs.~\eqref{eq:appA4_forward_posterior}--\eqref{eq:appA4_reverse_gaussian}),
their KL divergence has the closed form
{\small
\begin{equation}
\mathrm{KL}\!\left(
\mathcal{N}(\tilde{\mu}_s,\tilde{\beta}_s\mathbf{I})
\,\|\,\mathcal{N}(\mu_{\theta_{\mathrm{diff}}},\sigma_s^2\mathbf{I})
\right)
=
\frac{1}{2\sigma_s^2}\bigl\|\tilde{\mu}_s-\mu_{\theta_{\mathrm{diff}}}\bigr\|_2^2
+\frac{d}{2}\!\left(\frac{\tilde{\beta}_s}{\sigma_s^2}-1-\log\frac{\tilde{\beta}_s}{\sigma_s^2}\right),
\label{eq:appA4_kl_gaussian}
\end{equation}
}

where $d$ is the dimensionality of $Z_{s-1}$. The second term does not depend on $\theta_{\mathrm{diff}}$.

Substituting Eqs.~\eqref{eq:appA4_mu_eps_param} and \eqref{eq:appA4_mu_tilde_noise_form} into the mean difference yields
\begin{equation}
\tilde{\mu}_s(Z_s,Z_0)-\mu_{\theta_{\mathrm{diff}}}(Z_s,s)
=
\frac{\beta_s}{\sqrt{\alpha_s}\sqrt{1-\bar\alpha_s}}
\left(\epsilon_{\theta_{\mathrm{diff}}}(Z_s,s)-\epsilon\right).
\label{eq:appA4_mean_diff}
\end{equation}
Therefore,
\begin{equation}
\frac{1}{2\sigma_s^2}\bigl\|\tilde{\mu}_s-\mu_{\theta_{\mathrm{diff}}}\bigr\|_2^2
=
w_s \,\bigl\|\epsilon-\epsilon_{\theta_{\mathrm{diff}}}(Z_s,s)\bigr\|_2^2,
\qquad
w_s:=\frac{\beta_s^2}{2\sigma_s^2\,\alpha_s\,(1-\bar\alpha_s)}.
\label{eq:appA4_weighted_mse}
\end{equation}
Plugging Eq.~\eqref{eq:appA4_weighted_mse} into Eq.~\eqref{eq:appA4_L_sminus1_def} and using the reparameterization
$Z_s=\sqrt{\bar\alpha_s}Z_0+\sqrt{1-\bar\alpha_s}\epsilon$ yields
\begin{equation}
L_{s-1}(Z_0)
=
\mathbb{E}_{\epsilon\sim\mathcal{N}(0,\mathbf{I})}\Big[
w_s \,\bigl\|\epsilon-\epsilon_{\theta_{\mathrm{diff}}}(Z_s,s)\bigr\|_2^2
\Big]
+\mathrm{const}(s),
\label{eq:appA4_L_sminus1_final}
\end{equation}
where $\mathrm{const}(s)$ collects all terms independent of $\theta_{\mathrm{diff}}$ (including the covariance-only term
in Eq.~\eqref{eq:appA4_kl_gaussian}).

\paragraph{From per-step terms to the training objective.}
Summing over $s$ and dropping $\theta_{\mathrm{diff}}$-independent constants gives a diffusion training objective
equivalent to minimizing a weighted noise regression over noise levels. In practice, we sample a single timestep
$s\sim\mathrm{Unif}(\{1,\ldots,S\})$ per iteration and optimize the Monte Carlo estimate
{\small
\begin{equation}
\widehat{\mathcal{L}}_{\mathrm{diff}}^{\mathrm{train}}(\mathbf{X})
=
\mathbb{E}_{Z_0\sim q_{\theta_{\mathrm{enc}}}(Z_0\mid\mathbf{X})}
\;
\mathbb{E}_{s\sim \mathrm{Unif}(\{1,\ldots,S\}),\,\epsilon\sim\mathcal{N}(0,\mathbf{I})}
\Big[
w_s\,
\bigl\|\epsilon-\epsilon_{\theta_{\mathrm{diff}}}(Z_s,s)\bigr\|_2^2
\Big],
\end{equation}
}

which is exactly Eq.~\eqref{eq:diff_train_loss}.
\qed
\subsection{ Unbiasedness of single-timestep Monte Carlo training}
\label{app:proof_C5}

This appendix justifies the single-timestep training strategy used in Sec.~\ref{sec:efficient_diff} by showing that
sampling one diffusion timestep per iteration yields an unbiased Monte Carlo estimator of the
multi-step diffusion objective. This aligns with the standard diffusion training practice of sampling
$t$ uniformly. 

\paragraph{Setup.}
Let $S$ be the number of diffusion steps and let $F_s(Z_0)$ denote any per-timestep loss term
(e.g., a per-step KL term or its equivalent noise-prediction form).
We consider two equivalent ways to write the diffusion objective:

\smallskip
\noindent
\textbf{(i) Sum form:}
\begin{equation}
\mathcal{J}_{\Sigma}(Z_0) := \sum_{s=1}^{S} F_s(Z_0).
\label{eq:A5_sum_obj}
\end{equation}
\textbf{(ii) Expectation form (uniform over steps):}
\begin{equation}
\mathcal{J}_{\mathbb{E}}(Z_0) := \mathbb{E}_{s\sim \mathrm{Unif}(\{1,\ldots,S\})}\!\big[F_s(Z_0)\big]
= \frac{1}{S}\sum_{s=1}^{S} F_s(Z_0).
\label{eq:A5_exp_obj}
\end{equation}
The two objectives differ only by a constant scaling:
$\mathcal{J}_{\Sigma}(Z_0)=S\,\mathcal{J}_{\mathbb{E}}(Z_0)$.

\paragraph{Unbiased estimator under uniform sampling.}
Let $\hat{s}\sim \mathrm{Unif}(\{1,\ldots,S\})$ be a single sampled timestep.
Define the single-sample estimators
\begin{equation}
\widehat{\mathcal{J}}_{\mathbb{E}}(Z_0) := F_{\hat{s}}(Z_0),
\qquad
\widehat{\mathcal{J}}_{\Sigma}(Z_0) := S\,F_{\hat{s}}(Z_0).
\label{eq:A5_uniform_estimators}
\end{equation}

\begin{lemma}[Unbiasedness under uniform sampling]
\label{lem:A5_unbiased_uniform}
For any fixed $Z_0$, the estimators in Eq.~\eqref{eq:A5_uniform_estimators} satisfy
\begin{equation}
\mathbb{E}_{\hat{s}}\big[\widehat{\mathcal{J}}_{\mathbb{E}}(Z_0)\big]=\mathcal{J}_{\mathbb{E}}(Z_0),
\qquad
\mathbb{E}_{\hat{s}}\big[\widehat{\mathcal{J}}_{\Sigma}(Z_0)\big]=\mathcal{J}_{\Sigma}(Z_0).
\end{equation}
\end{lemma}

\begin{proof}
By definition of the uniform distribution and linearity of expectation,
\begin{equation}
\mathbb{E}_{\hat{s}}\big[F_{\hat{s}}(Z_0)\big]
=
\frac{1}{S}\sum_{s=1}^{S}F_s(Z_0)
=
\mathcal{J}_{\mathbb{E}}(Z_0),
\end{equation}
which proves the first identity. Multiplying both sides by $S$ yields
$\mathbb{E}_{\hat{s}}[S F_{\hat{s}}(Z_0)]=\sum_{s=1}^{S}F_s(Z_0)=\mathcal{J}_{\Sigma}(Z_0)$,
proving the second identity.
\end{proof}

\paragraph{General (importance) sampling over timesteps.}
More generally, let $p(s)$ be any probability distribution on $\{1,\ldots,S\}$ with $p(s)>0$ for all $s$.
If $\hat{s}\sim p(\cdot)$, then
\begin{equation}
\widehat{\mathcal{J}}_{\Sigma}^{\mathrm{IS}}(Z_0) := \frac{F_{\hat{s}}(Z_0)}{p(\hat{s})}
\label{eq:A5_is_estimator}
\end{equation}
is an unbiased estimator of the sum objective:
\begin{equation}
\mathbb{E}_{\hat{s}\sim p}\!\left[\widehat{\mathcal{J}}_{\Sigma}^{\mathrm{IS}}(Z_0)\right]
=
\sum_{s=1}^{S} p(s)\,\frac{F_s(Z_0)}{p(s)}
=
\sum_{s=1}^{S}F_s(Z_0)
=
\mathcal{J}_{\Sigma}(Z_0).
\end{equation}
Similarly, an unbiased estimator for the expectation objective in Eq.~\eqref{eq:A5_exp_obj} is
$\widehat{\mathcal{J}}_{\mathbb{E}}^{\mathrm{IS}}(Z_0):=\frac{1}{S}\frac{F_{\hat{s}}(Z_0)}{p(\hat{s})}$.

\paragraph{Implication for Sec.~\ref{sec:efficient_diff}.}
In Sec.~\ref{sec:efficient_diff} we define the diffusion training loss directly in expectation form over a uniformly sampled timestep,
which makes the single-timestep estimator unbiased by Lemma~\ref{lem:A5_unbiased_uniform}.
The same argument holds when $F_s(Z_0)$ is instantiated as the weighted $\epsilon$-prediction loss
derived in Appendix~\ref{app:proof_C4}, since the unbiasedness depends only on the timestep sampling and not on the
specific functional form of $F_s$.

\section{Performance Metrics}
\label{app:met}

\begingroup
\setlength{\abovedisplayskip}{2pt}
\setlength{\belowdisplayskip}{2pt}
\setlength{\abovedisplayshortskip}{2pt}
\setlength{\belowdisplayshortskip}{2pt}

We report two uncertainty/reliability metrics. The average posterior entropy is the mean Shannon entropy of the predicted posterior over edges:
\vspace{-0.3em}
\begin{equation}
\bar{H}
=
\frac{1}{|\mathcal{E}|}\sum_{e\in\mathcal{E}} H(p_e),
\qquad
H(p_e)
=
-\sum_{y\in\mathcal{Y}} p_e(y)\log p_e(y).
\label{eq:posterior_entropy}
\end{equation}
\vspace{-0.6em}

Lower $\bar{H}$ indicates more concentrated, and thus more certain, posteriors. 
We also compute the expected calibration error (ECE) with $B=10$ confidence bins:
\vspace{-0.3em}
\begin{equation}
\begin{aligned}
\mathrm{ECE}
&=
\sum_{b=1}^{B}
\frac{|I_b|}{N}
\left|
\mathrm{Acc}(I_b)-\mathrm{Conf}(I_b)
\right|, \\
\mathrm{Acc}(I_b)
&=
\frac{1}{|I_b|}
\sum_{i\in I_b}\mathbf{1}(\hat{y}_i=y_i),
\qquad
\mathrm{Conf}(I_b)
=
\frac{1}{|I_b|}
\sum_{i\in I_b}\hat{c}_i .
\end{aligned}
\label{eq:ece}
\end{equation}
\vspace{-0.6em}

Here, $I_b$ contains samples whose confidence $\hat{c}_i=\max_y p_i(y)$ falls in bin $b$. 
Lower ECE indicates better calibration.

\begin{table}[h]
\centering
\caption{Comparison between MLP and Transformer denoisers in the diffusion module. Results are reported as AUROC (\%). Bold indicates the better result under the same dynamics and dataset.}
\label{tab:denoiser_arch}
\setlength{\tabcolsep}{4pt}
\begin{tabular}{llrrrrr}
\hline\hline
Dynamics & Denoiser & VN & FW & BN & GRN & CRNA \\
\midrule
SP & MLP    & \textbf{95.17} & \textbf{81.51} & 99.52          & 90.58          & \textbf{85.65} \\
SP & Trans. & 94.75          & 81.17          & \textbf{99.60} & \textbf{91.41} & 84.78          \\
\midrule
NS & MLP    & 91.50          & \textbf{55.73} & 99.75          & 77.38          & \textbf{51.80} \\
NS & Trans. & \textbf{95.31} & 51.35          & \textbf{99.79} & \textbf{82.55} & 50.77          \\
\hline\hline
\end{tabular}
\end{table}

\section{Diffusion hyperparameters}
\label{app4}
We adopt an edge-wise DDPM prior over the encoder logits and set the diffusion horizon to $S{=}100$.
We use a linear noise schedule with $\beta_s$ linearly interpolated from $10^{-4}$ to $2\times 10^{-2}$.
The denoiser $\epsilon_\theta(\mathbf{z}_s,s)$ is implemented as a three-layer MLP that takes the concatenation of $\mathbf{z}_s\in\mathbb{R}^{K}$ and a sinusoidal time embedding of dimension $64$ as input, uses two hidden layers of width $128$ with SiLU activations and dropout $0.3$, and outputs a prediction in $\mathbb{R}^{K}$.
During training, we sample diffusion steps uniformly from $\{2,\ldots,s_{\max}\}$, averaging the resulting weighted noise-prediction objective.
Following the benchmark setting, we use $\sigma_s^2{=}\beta_s$ when constructing the per-step weight.
For the decoder input, we apply a decisive one-step residual refinement at $s_{\mathrm{ref}}{=}30$ with residual scale $\gamma{=}0.1$, using fixed Gaussian noise; we additionally clip the refined prediction with $\lVert \hat{\mathbf{z}}_0\rVert_\infty \le 10$ for stability.
To further examine the effect of the denoiser architecture in the diffusion module, we additionally replace the default MLP denoiser with a Transformer-based denoiser and report the comparison in Table~\ref{tab:denoiser_arch}. 
The two denoiser choices achieve comparable structure inference performance, while the Transformer variant introduces higher computational cost due to its attention operations. 
Therefore, we use the MLP denoiser as the default implementation in our main experiments.

{\small
\begin{table}[h]
\setlength{\tabcolsep}{2pt}
\caption{Training time per epoch (in seconds) under different graph priors (Uniform, Fixed, and Diff-prior) across five datasets.}
\label{app:runtime_table}
\begin{tabular}{cccccccc}
\hline\hline
\multicolumn{8}{c}{\textbf{Spring}}\\ \hline
\multicolumn{1}{c|}{\textbf{Model}}                & \multicolumn{1}{c|}{\textbf{prior}}                    & \textbf{VN\_15}      & \textbf{FW\_15}      & \textbf{BN\_15}      & \textbf{GRN\_15}     & \multicolumn{1}{c|}{\textbf{CRNA\_15}}    & \textbf{AVG}           \\ \hline
\multicolumn{1}{c|}{\multirow{3}{*}{\textbf{NRI}}} & \multicolumn{1}{c|}{\textbf{Uniform}}                  & 133                  & 131                  & 93                   & 96                   & \multicolumn{1}{c|}{62}                   & 103                    \\
\multicolumn{1}{c|}{}                              & \multicolumn{1}{c|}{\textbf{Fixed}}                    & 124                  & 137                  & 124                  & 98                   & \multicolumn{1}{c|}{82}                   & 113                    \\
\multicolumn{1}{c|}{}                              & \multicolumn{1}{c|}{\textbf{Diff-prior}}               & 175                  & 104                  & 112                  & 114                  & \multicolumn{1}{c|}{85}                   & 118                    \\ \hline
\multicolumn{1}{c|}{\multirow{3}{*}{\textbf{ACD}}} & \multicolumn{1}{c|}{\textbf{Uniform}}                  & 175                  & 82                   & 115                  & 189                  & \multicolumn{1}{c|}{111}                  & 134.4                  \\
\multicolumn{1}{c|}{}                              & \multicolumn{1}{c|}{\textbf{Fixed}}                    & 92                   & 119                  & 91                   & 91                   & \multicolumn{1}{c|}{89}                   & 96.4                   \\
\multicolumn{1}{c|}{}                              & \multicolumn{1}{c|}{\textbf{Diff-prior}}               & 115                  & 97                   & 101                  & 98                   & \multicolumn{1}{c|}{102}                  & 102.6                  \\ \hline
\multicolumn{1}{c|}{\multirow[c]{2}{*}{\textbf{MPM}}} & \multicolumn{1}{c|}{{\textbf{Uniform}}} & {118} & {87}  & {165} & {88}  & \multicolumn{1}{c|}{{114}} & {114.4}     \\
\multicolumn{1}{c|}{}                              & \multicolumn{1}{c|}{\textbf{Diff-prior}}               & 179                  & 171                  & 104                  & 167                  & \multicolumn{1}{c|}{179}                  & 160                    \\ \hline\hline

\multicolumn{8}{c}{\textbf{Netsims}}\\ \hline
\multicolumn{1}{c|}{\textbf{Model}}                & \textbf{prior}                                         & \textbf{VN\_15}      & \textbf{FW\_15}      & \textbf{BN\_15}      & \textbf{GRN\_15}     & \textbf{CRNA\_15}                         & \textbf{AVG}           \\ \hline
\multicolumn{1}{c|}{\multirow{3}{*}{\textbf{NRI}}} & \multicolumn{1}{c|}{\textbf{Uniform}}                  & 101                  & 78                   & 110                  & 71                   & \multicolumn{1}{c|}{65}                   & 85                     \\
\multicolumn{1}{c|}{}                              & \multicolumn{1}{c|}{\textbf{Fixed}}                    & 106                  & 102                  & 77                   & 107                  & \multicolumn{1}{c|}{65}                   & 91.4                   \\
\multicolumn{1}{c|}{}                              & \multicolumn{1}{c|}{\textbf{Diff-prior}}               & 99                   & 100                  & 103                  & 102                  & \multicolumn{1}{c|}{78}                   & 96.4                   \\ \hline
\multicolumn{1}{c|}{\multirow{3}{*}{\textbf{ACD}}} & \multicolumn{1}{c|}{\textbf{Uniform}}                  & 85                   & 92                   & 83                   & 82                   & \multicolumn{1}{c|}{69}                   & 82.2                   \\
\multicolumn{1}{c|}{}                              & \multicolumn{1}{c|}{\textbf{Fixed}}                    & 101                  & 93                   & 91                   & 104                  & \multicolumn{1}{c|}{113}                  & 100.4                  \\
\multicolumn{1}{c|}{}                              & \multicolumn{1}{c|}{\textbf{Diff-prior}}               & 92                   & 87                   & 93                   & 153                  & \multicolumn{1}{c|}{89}                   & 102.8                  \\ \hline
\multicolumn{1}{c|}{\multirow[c]{2}{*}{\textbf{MPM}}} & \multicolumn{1}{c|}{{\textbf{Uniform}}} & {148} & {154} & {152} & {153} & \multicolumn{1}{c|}{{152}} & {151.8}     \\
\multicolumn{1}{c|}{}                              & \multicolumn{1}{c|}{\textbf{Diff-prior}}               & 109                  & 111                  & 152                  & 199                  & \multicolumn{1}{c|}{144}                  & 143                    \\ \hline\hline
\end{tabular}
\end{table}
}

\begin{table}[h]
\centering
\small
\caption{Runtime and memory comparison.}
\label{tab:runtime_memory}
\begin{tabular}{lrrrr}
\hline\hline
 & Uni-30 & Diff-30 & Uni-50 & Diff-50 \\
\midrule
Time (s/epoch) & 250 & 260 & 675 & 712 \\
Mem (GB)       & 4.65 & 4.74 & 12.73 & 12.80 \\
\hline\hline
\end{tabular}%
\end{table}

\section{Runtime Statistics}
\label{app:runtime}
Table~\ref{app:runtime_table} summarizes the per-epoch average training time for each baseline across three dynamic backbones and five network settings. 
Despite introducing additional parameters, Diff-prior does not incur noticeable training-time overhead compared to Uniform/Fixed, and exhibits comparable per-epoch runtime under identical hardware and training configurations.

We further report the runtime and memory usage under larger graph sizes in Table~\ref{tab:runtime_memory}. 
When increasing the number of nodes from 30 to 50, the absolute runtime and memory naturally increase because the number of candidate edges grows with graph size. 
However, the additional overhead introduced by Diff-prior remains marginal compared with the corresponding Uniform-prior baseline. 
For the 30-node setting, Diff-prior only increases the runtime from 250 to 260 seconds per epoch and the memory usage from 4.65 GB to 4.74 GB. 
For the 50-node setting, the runtime increases from 675 to 712 seconds per epoch, while the memory usage only changes from 12.73 GB to 12.80 GB. 
These results show that Diff-prior introduces only a slight extra runtime  extra memory cost overhead, even when the graph size becomes larger.

\section{Step-wise Performance on NRI}
\label{app:nri_steps}

Table~\ref{tab:nri_steps} reports the step-wise AUROC performance of the NRI backbone on the FW dataset under SP and NS dynamics. 
The purpose of this comparison is to examine whether multiple refinement steps are necessary for the diffusion-based calibration. 
For FW-SP, the best result is achieved at step-1, while the performances from step-1 to step-10 remain very close. 
For FW-NS, although step-3 gives the highest AUROC, the gap between step-1 and step-3 is relatively small compared with the improvement already obtained after the first refinement step. 
These results indicate that most of the calibration benefit can be achieved with a single refinement step, and using more steps only brings marginal additional gains. 
Therefore, in our main experiments, we adopt one-step refinement as the default setting, which preserves the effectiveness of Diff-prior while avoiding additional iterative computation and runtime overhead.

\begin{table}[h]
\centering
\caption{Step-wise performance comparison on NRI.}
\label{tab:nri_steps}
\begin{tabular}{lrrrrrr}
\hline\hline
 & NRI & step-0 & step-1 & step-3 & step-5 & step-10 \\
\midrule
FW-SP & 81.09 & 80.91 & \textbf{81.51} & 81.39 & 81.22 & 81.23 \\
FW-NS & 51.78 & 50.60 & 55.73 & \textbf{56.89} & 52.24 & 51.02 \\
\hline\hline
\end{tabular}
\end{table}

\end{document}